\documentclass{article}

\usepackage[final]{neurips_2022}

\usepackage[utf8]{inputenc} 
\usepackage[T1]{fontenc}    
\usepackage{hyperref}       
\usepackage{url}            
\usepackage{booktabs}       
\usepackage{amsfonts}       
\usepackage{nicefrac}       
\usepackage{microtype}      
\usepackage{xcolor}         

\usepackage{amsfonts, dsfont, amssymb, amsthm, amsmath,graphicx,xcolor,algorithm}
\usepackage{semantic}
\usepackage{algpseudocode, cleveref}

\newcommand{\ignore}[1]{}

\definecolor{forestgreen}{rgb}{0.0, 0.27, 0.13}

\newtheorem{defn}{Definition}
\newtheorem{thm}{Theorem}
\newtheorem{prop}{Proposition}
\newtheorem{cor}{Corollary}[thm]
\newtheorem{lem}{Lemma}
\newtheorem{assum}{Assumption}

\mathlig{==}{\equiv}
\mathlig{=.}{\doteq}
\mathlig{:=}{\triangleq}
\mathlig{<<}{\ll}
\mathlig{>>}{\gg}
\mathlig{<>}{\neq}
\mathlig{<=}{\leq}
\mathlig{>=}{\geq}
\mathlig{<==}{\Leftarrow}
\mathlig{==>}{\Rightarrow}
\mathlig{<=>}{\Leftrightarrow}
\mathlig{<==>}{\iff}
\mathlig{<--}{\leftarrow}
\mathlig{-->}{\rightarrow}
\mathlig{<->}{\leftrightarrow}
\mathlig{+-}{\pm}
\mathlig{-+}{\mp}
\mathlig{...}{\dots}
\mathlig{!=}{\stackrel{!}{=}}

\newif\ifshowanswer    

\newcommand{\isitthree}[1]
{
  \ifnum#1=3
    number #1 is 3
  \else
    number #1 is not 3
  \fi
}

\newcommand{\Var}{\mathrm{Var}}

\newcommand{\be}{\begin{equation}}
\newcommand{\ee}{\end{equation}}

\renewcommand\P{{\mathds{P}}}
\newcommand\E{{\mathds{E}}}

\newcommand\CN{{\mathcal N}}

\usepackage{amsmath}
\usepackage{bbm}

\renewcommand{\epsilon}{\varepsilon}

\newcommand{\tbd}{\sqrt{\epsilon}}
\newcommand{\DKL}{\mathrm{D}_\mathrm{KL}}
\newcommand{\DTV}{\mathrm{D}_\mathrm{TV}}

\usepackage{essay-def}

\usepackage{stackengine}
\setstackEOL{\\}
\usepackage{tikz-cd}

\newcommand{\sign}{\texttt{sign}}

\newcommand{\ConWid}{b_{r}}

\newcommand{\stepa}[1]{\overset{\rm (a)}{#1}}
\newcommand{\stepb}[1]{\overset{\rm (b)}{#1}}
\newcommand{\stepc}[1]{\overset{\rm (c)}{#1}}

\renewcommand{\epsilon}{\varepsilon}

\title{Beyond the Best: 
Estimating Distribution Functionals in Infinite-Armed Bandits}

\author{%
  Yifei Wang \\
  Department of Electrical Engineering\\
  Stanford University\\
  Stanford, CA 94305 \\
  \texttt{wangyf18@stanford.edu} \\
  \And
  Tavor Z. Baharav \\
  Department of Electrical Engineering\\
  Stanford University\\
  Stanford, CA 94305 \\
  \texttt{tavorb@stanford.edu} \\
  \And
  Yanjun Han\\
  Institute for Data, Systems, and Society\\
  Massachusetts Institute of Technology\\
  Cambridge, MA 02142\\
  \texttt{yjhan@mit.edu}\\
  \And
  Jiantao Jiao\\
  Department of Electrical  Engineering and Computer Sciences and Department of Statistics\\
  University of California, Berkeley\\
  Berkeley, CA 94720\\
  \texttt{jiantao@eecs.berkeley.edu}\\
  \And
  David Tse\\
  Department of Electrical Engineering\\
  Stanford University\\
  Stanford, CA 94305 \\
  \texttt{dntse@stanford.edu}
}

\begin{document}

\maketitle
\begin{abstract}
In the infinite-armed bandit problem, each arm's average reward is sampled from an unknown distribution, and each arm can be sampled further to obtain noisy estimates of the average reward of that arm. Prior work focuses on identifying the best arm, i.e., estimating the maximum of the average reward distribution. We consider a general class of distribution functionals beyond the maximum, and propose unified meta algorithms for both the offline and online settings, achieving optimal sample complexities. We show that online estimation, where the learner can sequentially choose whether to sample a new or existing arm, offers no advantage over the offline setting for estimating the mean functional, but significantly reduces the sample complexity for other functionals such as the median, maximum, and trimmed mean. The matching lower bounds utilize several different Wasserstein distances. For the special case of median estimation, we identify a curious thresholding phenomenon on the indistinguishability between Gaussian convolutions with respect to the noise level, which may be of independent interest. 
\end{abstract}

\section{Introduction}
\vspace{-.2cm}
In the infinite-armed bandit problem~\citep{berry1997bandit}, at each time instance the learner can either sample an arm that has been previously observed, or sample from a new arm, whose average reward is drawn from an unknown distribution $F$. The learner's goal is to identify arms with large average reward, with the objective being achieving either small cumulative regret~\citep{berry1997bandit,wang2008algorithms,bonald2013two}, or small simple regret~\citep{carpentier2015simple}. This setting differs from the classical multi-armed bandit formulation as the number of observed arms is not fixed a priori and needs to be carefully chosen by the algorithm. 

We consider the problem of estimating some functional $g(F)$ of an underlying distribution $F$, as is illusrated in Figure \ref{fig:illus}. From this point of view, the classical infinite-armed bandit problem can be viewed as an \emph{online} sampling algorithm to estimate the \emph{maximum} of the distribution $F$. \footnote{To be precise, the objectives in infinite-armed bandit works~\citep{berry1997bandit,wang2008algorithms,bonald2013two,carpentier2015simple} are slightly different, minimizing simple or cumulative regret.} Once we cast the infinite-armed bandit problem in this manner, it immediately suggests several additional questions. For example, what about \emph{offline} sampling algorithms? Indeed, online sampling requires continual interactions with the environment which may be infeasible in certain applications, and recent work in online and offline reinforcement learning have demonstrated the significant value of both formulations~\citep{rashidinejad2021bridging, zhang2021reinforcement,schrittwieser2021online}. Additionally, it is worth estimating functionals beyond the maximum: in many practical scenarios, including mean estimation in single-cell RNA-sequencing \citep{zhang2020determining} and  Benjamini Hochberg (BH) threshold estimation in multiple hypothesis testing \citep{zhang2019adaptive}, we are interested in the mean, median (quantile), or trimmed mean of the underlying distribution $F$. The estimation of quantiles is similar to estimation of the BH threshold, as both depend on the order statistics of the underlying distribution. Estimating the median or trimmed mean has further applications in robust statistic for instance, maintaining the fidelity of an estimator in the presence of adversarial corruption or outliers. Another natural setting where such problems arise is in large-scale distributed learning \citep{son2012distributed}. Here, a server / platform wants to estimate how much
test-users like their newly released product. Users return a noisy realization of their affinity for the
product, and the platform can decide to pay the user further to spend more time with the product, to
test it further. For many natural objectives which are robust to a small fraction of adversarial users, e.g. trimmed mean, median, or quantile estimation, we see that our algorithm will enable estimation of the desired quantity to high accuracy while minimizing the total cost (number of samples taken). 
Since sampling is expensive, it is critical to identify the optimal method to collect samples, and identify the improvements afforded by adaptivity. 
For example, do online methods offer significant gains over offline methods? Are the fundamental limits of estimating the median and trimmed mean different from that of the maximum? 

\begin{figure}[!htbp]
    \centering
    \begin{tikzcd}[column sep=tiny]
   &&&\framebox[2\width]{$F$}
\arrow[lld,"\texttt{subsample}"']\arrow[rrd]\arrow[d]
&&&\\
   &\framebox[2\width]{$X_1$}\arrow[ld,"\texttt{noise}"']\arrow[d]\arrow[rd]
   &&\framebox[2\width]{$...$}\arrow[d]&&\framebox[2\width]{$X_n$}\arrow[d]\arrow[rd]\arrow[ld]
   &\\
\framebox[2\width]{$Y_{1,1}$}&\framebox[2\width]{$...$}& \framebox[2\width]{$Y_{1,m}$}&\framebox[2\width]{$...$}&\framebox[1.5\width]{$Y_{1,m-1}$}&\framebox[2\width]{$...$}& \framebox[2\width]{$Y_{n,m}$}
    \end{tikzcd}
    \caption{Problem setting. Level 0: underlying distribution $F(x)$. Level 1: unobserved samples $X_1,\dots,X_n\sim F(x)$. Level 2: noisy observations $Y_{i,j}\sim \mcN(X_i,1)$.}\label{fig:illus}
\end{figure}
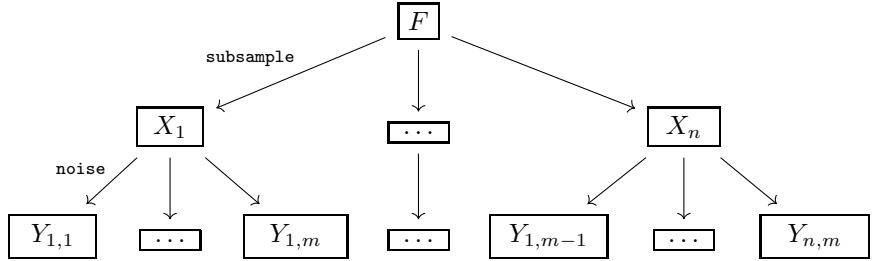

In this paper we initiate the study of distribution functional estimation in both online and offline settings and obtain both information theoretic limits and efficient algorithms for estimating the mean, median, trimmed mean, and maximum. We propose unified meta algorithms for both offline and online settings, and provide 
matching upper and lower bounds
for the sample complexity of estimating the aforementioned functionals in Table \ref{tab:functionals}. 

We also reveal new insights on the fundamental differences between the offline and online algorithms, as well as the fundamental differences between different functionals. To determine these sharp statistical limits, we use the Wasserstein-$2$ distance to upper bound the KL divergence in the offline setting, while the Wasserstein-$\infty$ distance is applied in the online setting instead. This approach leads to valid sample complexity lower bounds for general functionals $g$, which turn out to be tight for estimating the mean and maximum. However, a curious thresholding phenomenon, which is not captured by the previous approach and does not occur for the \emph{mean} and \emph{maximum}, appears in the \emph{median} and \emph{trimmed mean} analyses: the KL divergence does not change smoothly with the noise level and enjoys a phase transition after the noise level exceeds some threshold. This phenomenon calls for different treatments under different estimation targets and could be of independent interest.

\begin{table}[h]
    \centering
    \begin{tabular}{|c|c|c|c|c|c|} \hline
        Functional \rule{0pt}{4.1ex}& \shortstack{Offline \\ complexity} & \shortstack{Online \\ complexity} & Comments\\ \hline
        Mean \rule{0pt}{2.6ex}& $\Theta(\epsilon^{-2})$ & $\Theta(\epsilon^{-2})$ & No gain from online sampling\\ \hline
        Median \rule{0pt}{2.6ex}& $\Theta(\epsilon^{-3})$ & $\tilde\Theta(\epsilon^{-2.5})$ & Holds for any quantile not on the boundary\\ \hline
        Maximum \rule{0pt}{2.6ex}& $\Theta(\epsilon^{-(2+\beta)})$ & $\tilde\Theta(\epsilon^{-\max(\beta,2)})$ & Depends on the tail regularity $\beta$\\ \hline 
        Trimmed mean \rule{0pt}{2.6ex}& $\tilde\Theta(\epsilon^{-3})$ & $\tilde\Theta(\epsilon^{-2.5})$ & $g(F)=\E\{X | X \in [F^{-1}(\alpha), F^{-1}(1-\alpha)]\}$  \\\hline
    \end{tabular}
    \caption{Sample complexity of estimating different functionals $g(F)$, where $F$ is the cumulative distribution function (CDF) of the distribution to estimate. The trimmed mean result holds for a fixed $\alpha \in (0,1/2)$. Here $\epsilon$ is the target accuracy and we use ${\Theta}$ to denote the matching upper and lower bounds up to constants not depending on $\epsilon$.
    Additionally, we use $\tilde{\Theta},\gtrsim$, and $\lesssim$ to suppress constants and logarithmic factors in $\varepsilon$, and $\epsilon^c$ for any fixed $c$ arbitrarily close to zero.
    If $h(\epsilon) \lesssim f(\epsilon)$ and $f(\epsilon) \lesssim h(\epsilon)$ then we denote this as $f(\epsilon) \asymp h(\epsilon)$.
    For maximum estimation, we assume that the distribution satisfies $\P(X\geq F^{-1}(1)-\epsilon)\asymp \epsilon^{\beta}$. Other assumptions on $F$ are detailed in \Cref{sec:algs}. 
    }
    \label{tab:functionals}
    \vspace{-.4cm}
\end{table}

The rest of this paper is structured as follows.
In \Cref{sec:relatedWork} we discuss the relevant literature.
We then formulate our distribution functional estimation problem in \Cref{sec:prob_form}.
Our unified meta algorithms for the offline and online settings are presented in \Cref{sec:algs}, where we show the sample complexity upper bounds.
We present information theoretic lower bounds proofs via Wasserstein distance for the online and offline settings in \Cref{sec:lbs}, and discuss a special thresholding phenomenon arising in median estimation in \Cref{sec:medLB}.
\Cref{sec:conc} concludes this work. 

\subsection{Related works} \label{sec:relatedWork}\vspace{-.2cm}
The field of multi-armed bandits has seen broad interest and utility since its formalization in 1985 \citep{lai1985asymptotically}. 
Across clinical trials, multi-agent learning, online recommendation systems, and beyond \citep{lattimore2020bandit}, multi-armed bandits have proven to be an excellent framework for modeling and solving complex tasks regarding exploration in an unknown environment.
In the classical multi-armed bandit setting we have a set of $n$ distributions, where the player sequentially pulls one arm per round and observes a sample drawn from the associated reward distribution. In the infinite-armed bandit setting \citep{berry1997bandit}, the average arm reward for each arm is sampled i.i.d. from an unknown distribution, i.e., we have infinitely many available arms. There are many possible objectives that can be formulated in this online learning problem, from cumulative/simple regret minimization \citep{wang2008algorithms,bonald2013two,carpentier2015simple,li2017infinitely} to identification tasks (for example identifying an arm whose average reward is $\epsilon$ close to the largest average reward) \citep{aziz2018pure, chaudhuri2017pac, chaudhuri2019pac}.
Many works have studied best-arm identification, and we now have essentially matching instance-dependent upper and lower bounds \citep{jamieson2014best,kaufmann2016complexity}. One could also use the average reward estimate of the identified best arm to estimate the maximum of the average reward distribution in the infinite-armed bandit setting~\citep{carpentier2015simple, aziz2018pure, chaudhuri2017pac, chaudhuri2019pac}. 

From a statistical perspective, the sample complexity in the offline setting is closely related to deconvolution distribution estimation \citep{cordy1997deconvolution, wasserman2004all, hall2008estimation, delaigle2008deconvolution, dattner2011deconvolution}. Nevertheless, these previous works mainly focus on the expected L2 difference between the underlying distribution function and its estimation. This simplified setting does not allow for consideration of the trade-off inherent in our setting between the number of points and the (variable) number of observations per point. Additionally, these past works did not calculate the specific sample complexity for more general functionals like quantile and trimmed mean. Since the noise is treated as fixed and uniform, there has been no study of the online setting where adaptive resampling can enable dramatic sample complexity improvements. In particular, the challenge is that we have noisy observations, which makes deriving lower bounds even in offline cases a significant challenge that has not been dealt with in the past, let alone analyzing the online case. The dramatic performance gains afforded by adaptive resampling for functional estimation, combined with its lack of formal study, motivates the focus of this work.

\vspace{-.1cm}
\section{Problem formulation}\label{sec:prob_form}\vspace{-.2cm}
We are interested in estimating the distribution functional $g(F)\in \mathbb{R}$ of an underlying distribution with cumulative distribution function (CDF) $F$. We study a class of indicator-based functionals $g$ defined as follows.
\begin{defn}[Indicator-based functionals] \label{assum:gs}
The functional $g$ can be represented as 
\begin{equation}
    g(F) = \E \left[ X | X \in S(F) \right]
\end{equation}
for some set $S(F)$, where $X\sim F$. The set $S(F)$ is defined as follows:
\begin{equation}
    S(F)=[F^{-1}(\alpha_1),F^{-1}(\alpha_2)], 0\leq \alpha_1 \leq \alpha_2 \leq 1. 
\end{equation}
\end{defn}
We denote $S(F)$ by $S$ throughout this work when $F$ is clear from context.
This class encompasses many natural functionals of interest, which we formulate in \Cref{tab:functionals2}. 
In Appendix \ref{app:extension}, we discuss extending our results to more general functionals, and show that our approach can extend to smooth reweighting functions $h(X)$ and more complex sets $S$.

\begin{table}[h]
    \centering
    \begin{tabular}{|c|c|c|c|c|} \hline
        Functional \rule{0pt}{2.4ex}& $g(F)$ & $\alpha_1$ & $\alpha_2$&Comment\\ \hline
        Mean \rule{0pt}{2.4ex}& $\E[X]$&0&1&\\ \hline
        Quantile \rule{0pt}{2.4ex} & $F^{-1}(\alpha)$ & $\alpha$ & $\alpha$&$\alpha\in(0,1)$, e.g. $\alpha=1/2$ for median\\ \hline
        Maximum \rule{0pt}{2.4ex}& $F^{-1}(1)$ & $1$ & $1$&$\alpha_1=\alpha_2=0$ for minimum\\ \hline
        Trimmed mean \rule{0pt}{2.9ex}& 
        $\E[X|F(X)\in[\alpha_1,\alpha_2]]$
        & $\alpha$ & $1-\alpha$&$\alpha\in(0,1/2)$\\\hline
    \end{tabular}
    \caption{Indicator-based functionals.}
    \label{tab:functionals2}
\end{table}

As in the infinite-armed bandit setting, we only have access to noisy observations of samples drawn from the distribution with CDF $F$. We can either choose to sample from a point $X$ which we already have some noisy observations of, or draw a new point $X$ from $F$.
We then observe $Y=X+Z$, where $Z \sim \mcN(0,1)$ is independent of everything observed so far.  

In this paper we characterize the online and offline sample complexities of these problems, and in \Cref{sec:algs} propose online and offline algorithms achieving them.
For $\epsilon>0$ and $\delta\in(0,1)$, we call an estimator $\hat{G}$ an $(\epsilon,\delta)$-PAC approximation of $g(F)$ if
$
\P(|\hat{G}-g(F)|> \epsilon)\le \delta.    
$

\section{Offline and online algorithms} \label{sec:algs}
\subsection{Offline estimation algorithms}\label{ssec:offline}

We study a special class of offline algorithms, which uniformly obtain observations of the points following the underlying distribution. To be precise, based on prior information regarding the distribution in question, $F$, it will choose an appropriate number of points $n$ and number of samples per point $m$ to obtain an $(\epsilon,\delta)$-PAC approximation of $g(F)$. Specifically, the latent variables $X_1,\dots,X_n$ are drawn from $F$, and our observations $\{Y_{i,j}\}_{j=1}^m$ are drawn i.i.d. from $\mcN(X_i,1)$, independently for each $i$.
For $i\in[n]$, denote $\hat X_i=m^{-1}\sum_{j=1}^mY_{i,j}$ as the empirical mean of the observations for arm $i$.
Then, we can write $\hat X_i=X_i+\hat{Z}_i$ where $\hat Z_i\sim \mcN(0,1/m)$, independent across $i$. 
Define $S_n:= \{ i : X_{(\lfloor\alpha_1 n\rfloor)} \le X_i \le X_{(\lfloor\alpha_2 n\rfloor)}\}$ as the set of arms relevant for estimating the functional $g$, and define our $n$ sample estimate of $g$ as $g_n(\hat{X}_1,\hdots,\hat{X}_n) := |S_n|^{-1} \sum_{i \in S_n} \hat{X}_i$. 
Here $X_{(i)}$ denotes the $i$-th order statistic, that is the $i$-th smallest entry in $X_1,\dots,X_n$. Then, $G_{n,m} = g_n(\hat{X}_1,\hdots,\hat{X}_n)$, where each $X_i$ has been sampled $m$ times, serves as a natural estimator for $g(F)$ from the noisy observations. 
With this, we can state the following theorem:
\begin{thm}[Offline PAC sample complexity]\label{thm:gen_uni}
An $(\epsilon,\delta)$-PAC offline uniform-sampling-based algorithm for estimating $g(F)$ requires $\Theta(nm)$ samples where $n,m$ depend on $\epsilon,\delta$, the functional $g$, and information about $F$, with orderwise dependence on $\epsilon$ detailed in \Cref{tab:uniform}.
\end{thm}

For the rest of this section, we discuss in greater detail our assumptions on the underlying distribution. We defer the proofs and calculations for $n$ and $m$ to \Cref{app:offlinenm}, as well as discussion regarding the trimmed mean to \Cref{app:tm}.

\begin{table}[h]
    \centering
    \begin{tabular}{|c|c|c|}
    \hline
        Functional &$m$&$n$  \\\hline
        Mean & $\Theta(1)$&$\Theta(\epsilon^{-2})$\\\hline
        Median & $\Theta(\epsilon^{-1})$&$\Theta(\epsilon^{-2})$\\\hline
        Maximum & $\Theta(\epsilon^{-2})$&$\Theta(\epsilon^{-\beta})$\\\hline
        Trimmed mean \rule{0pt}{2.4ex}& $\Theta\left(\epsilon^{-1}\log \left(\epsilon^{-1}\right)\right)$&$\Theta(\epsilon^{-2})$\\\hline
    \end{tabular}
    \caption{Choice of $(m,n)$ for estimating different functionals to accuracy $\epsilon$.
    }
    \label{tab:uniform}
    \vspace{-.5cm}
\end{table}

\subsubsection{Mean}
To guarantee that the empirical mean is a good estimator for the true mean, we impose assumptions on the tail of the distribution $F$:
\begin{assum}\label{assum:mean}
The distribution $F$ satisfies  $\Var_{X\sim F}[X]\leq c$. 
\end{assum}

\Cref{assum:mean} ensures that estimation of the mean of the distribution can be accomplished with finite samples. The following proposition gives the sample complexity of the offline algorithm. 
\begin{prop}\label{prop:uni_ub_mean}
Suppose that Assumption \ref{assum:mean} is satisfied. By choosing $m=1$ and $n\geq \delta^{-1}(1+c)\epsilon^{-2}$, the estimator $G_{n,m}$ is an $(\epsilon,\delta)$-PAC approximation of $g(F)$. Thus, the offline algorithm requires $O(\epsilon^{-2})$ samples.
\end{prop}

\subsubsection{Median}
For median estimation we require different assumptions than the mean, as listed below.
\begin{assum}\label{assum:med}
There exist constants $c_1,c_2>0$ such that
\begin{itemize}
    \item $F'(x)\ge c_1$ for $|x-\mathrm{median}(F)| \lesssim \epsilon$.
    \item $|F''(x)| \le c_2$ for $|x-\mathrm{median}(F)| \lesssim \sqrt{\epsilon}$.
\end{itemize}
\end{assum}
The first assumption ensures that the median of $F$ is unique. The second assumption precludes the distribution from being dumbbell-shaped (very little mass near the median), in which case  estimating the true median is meaningless and can be arbitrarily difficult.
The following proposition gives a suitable choice of $(n,m)$ for providing an $(\epsilon,\delta)$-PAC approximation of $g(F)$.
\begin{prop}\label{prop:med_ub}
Suppose that Assumption \ref{assum:med} holds. Then, by choosing $m\geq {4(c_2+1)}/(c_1\varepsilon)$ and $n\geq {28\log(1/\delta)}/(c_1\varepsilon)^2$, the estimator $G_{n,m}$ is an $(\epsilon,\delta)$-PAC approximation of $g(F)$. Thus, the offline algorithm requires $O(\epsilon^{-3})$ samples. 
\end{prop}

\subsubsection{Maximum}
For maximum estimation, we require an assumption on the tail of $F$ as is common in the infinite-armed bandit literature.
\begin{assum}\label{assum:max}
There exist constants $0<c_1<c_2$ and $\beta>0$ such that
\begin{itemize}
    \item $1-F(F^{-1}(1)-t)\in[c_1 t^\beta,c_2 t^\beta]$, for all $0\leq t\lesssim \epsilon$.
\end{itemize}
\end{assum}
This assumption is also known as the $\beta$-regularity of $F$ around $F^{-1}(1)$, see \citep{wang2008algorithms}.  We present a suitable choice of $(n,m)$ in the following proposition.
\begin{prop}\label{prop:max_ub}
Suppose that Assumption \ref{assum:max} holds. By choosing $n\geq c_1^{-1}2^\beta \epsilon^{-\beta}\log(2/\delta)$ and $m\geq 4\epsilon^{-2}\log(2n/\delta)$, the estimator $G_{n,m}$ is an $(\epsilon,\delta)$-PAC approximation of $g(F)$. Therefore, the offline algorithm requires $O(\epsilon^{-\beta-2})$ samples.
\end{prop}

\subsection{Online estimation algorithm}

We now present our general algorithm (\Cref{alg:genalg}), an elimination-based $(\epsilon,\delta)$-PAC algorithm that efficiently estimates $g(F)$, where $g$ is a known input functional and $F$ is an unknown distribution from which we are able to sample $X_i$ independently, and observe noisy observations $Y_{i,j}$ of $X_i$. 

In order to exploit the Bayesian nature of the problem, we analyze the algorithm in two parts.
First, we use the fact that our arms are drawn from a common distribution to find some $n,m$ as in \Cref{thm:gen_uni} such that the plug-in estimator $G_{n,m}$ will be an $(\epsilon/2,\delta/2)$-PAC approximation of $g(F)$.
Second, we show that our adaptive algorithm is an $(\epsilon/2,\delta/2)$-PAC approximation of $G_{n,m}$, but is able to accomplish this using significantly fewer than $n\times m$ samples.

Notationally, we denote by $\hat{\mu}(r)$ the estimated mean vector of all arms at round $r$, and denote the $i$-th entry of this vector by $\hat{\mu}_i(r)$. 
We have that with high probability each arm's mean estimate stays within its width $b_r=2^{-r}$ confidence interval for each round $r$.
To analyze our algorithm, we denote $\mu_1^\mathrm{uni},\dots,\mu_n^\mathrm{uni}$ as the estimates of $X_1,\dots,X_n$ generated by the offline algorithm after sampling each arm $m$ times.
Then, we see that for the offline algorithm the arms relevant for the estimation task and the corresponding $n$ sample estimator are
\begin{equation}
    S_n = \left\{ i : \mu_{(\lfloor \alpha_1 n\rfloor)}^\mathrm{uni} \le \mu_i^\mathrm{uni} \le \mu_{(\lfloor \alpha_2 n\rfloor)}^\mathrm{uni}\right\}, \quad G_{n,m} = \frac{1}{|S_n|} \sum_{i\in S_n} \mu_i^\mathrm{uni}.
\end{equation}
Here $S_n$ indicates the arms that the offline algorithm believes are in $S$.
We show that our online algorithm is able to efficiently estimate the set $S_n$ as $\hat S_n$,
determining whether or not arms are in $S_n$, sampling these arms in $S_n$ sufficiently, and returning a plug-in estimator.
By construction each arm is only pulled by the adaptive algorithm at most $m$ times, as we know from the analysis of the offline algorithm that for the utilized $n,m$, if each arm is pulled $m$ times then the output is an $(\epsilon/2,\delta/2)$-PAC estimate of $g(F)$.
Thus, the online algorithm's objective is essentially emulating the output of the offline algorithm, for which it only needs to sample any arm at most $m$ times.

\begin{algorithm}[t]
\begin{algorithmic}[1]
\caption{\label{alg:genalg} \texttt{Meta Algorithm}}
\State \textbf{Input:} target accuracy $\epsilon$, error probability $\delta$, functional $g$ parameterized by $(\alpha_1,\alpha_2)$
\State Compute $(n,m)$ for $(\epsilon/2,\delta/2)$-PAC estimation of $g(F)$ based on \Cref{thm:gen_uni}
\State Construct active set $A_1=[n]$, and define $b_0=1$ and $t_0=0$
\For{$r=1,2,\dots$}
\State Define $b_r=2^{-r}$ and $t_r=\min(m,\lceil8\ConWid^{-2}\log(16n\log(m)/\delta)\rceil)$
\State Pull each arm in $A_r$ for $t_r-t_{r-1}$ times, construct $\hat{\mu}(r)$
\State Compute $A_{r+1}=\left\{ i : |\hat{\mu}_i(r)-\hat{\mu}_{(\flr{\alpha_1 n})}(r)| \le b_r \text{ or }  |\hat{\mu}_i(r)-\hat{\mu}_{(\flr{\alpha_2 n})}(r)| \le b_r \right\}$
\If{$t_r\equiv m$}
\State \textbf{Break}, exit For loop
\EndIf
\EndFor
\State Construct $\hat{S}_n = \{i : \hat{\mu}_{(\flr{\alpha_1 n})}(r) \le \hat{\mu}_i(r) \le \hat{\mu}_{(\flr{\alpha_2 n})}(r)  \}$
\If{$\alpha_1==\alpha_2$}
\State \Return $\frac{1}{|\hat{S}_n|} \sum_{i \in \hat{S}_n} \hat{\mu}_i(r)$
\Else
\State Draw one observation from each $i \in \hat{S}_n$, construct $\tilde{\mu}_i$
\State \Return $\frac{1}{|\hat{S}_n|} \sum_{i \in \hat{S}_n} \tilde{\mu}_i$
\EndIf
\end{algorithmic}
\end{algorithm}

Note that when $\alpha_1\neq \alpha_2$, we have many samples $X_i$ that are within $S$, with $|S_n| \ge \flr{n(\alpha_2-\alpha_1)}$.
In order to avoid issues of dependence, we discard all previous samples (as arms in $\hat{S}_n$ will have been sampled different numbers of times), and see that since we have $\Theta(n)$ arms in $\hat{S}_n$ we can construct a sufficiently accurate estimate by sampling each arm in $\hat S_n$ once.
Algorithmically, we denote this as obtaining one fresh observation and constructing $\tilde{\mu}_i$.


To upper bound the sample complexity of our algorithm, we see that each arm only needs to be sampled to determine whether it is in $S_n$ or not.
As we show in \Cref{app:ubProofs}, the number of samples $N(i)$ needed for point $X_i$ satisfies
\begin{equation}
    N(i) \le \min\left(m, \frac{256\log \left(\frac{16n \log m}{\delta} \right)}{\left[\text{dist}(X_i,\partial\; \text{Conv}(\{{\mu}^\mathrm{uni}_i:i\in S_n\}))\right]^2}\right),
\end{equation}
with probability at least $1-\delta/4$ for all arms simultaneously, where $\partial A$ denotes the boundary of a set $A$, $\text{Conv}(A)$ denotes the convex hull of a set $A$, and $\text{dist}(X,A) = \min_{a \in A} |X-a|$. In the limit as $\epsilon \to 0$ we show that $\mu_{(\flr{\alpha_1n})}^\textrm{uni} \to F^{-1}(\alpha_1)$ (similarly with $\mu_{(\flr{\alpha_2n})}^\textrm{uni}$).
This allows us to state the following theorem regarding the expected sample complexity of \Cref{alg:genalg} with respect to the distribution's relevant set of values $S$ rather than the estimated indices $S_n$.

\begin{thm}[Meta algorithm]\label{thm:gen_ada}
For a functional $g$ satisfying \Cref{assum:gs}, \Cref{alg:genalg} provides an $(\epsilon,\delta)$-PAC estimate of $g(F)$ with $M$ samples when given the requisite inputs. Here $m$ and $n$ are calculated as in Theorem \ref{thm:gen_uni}, and the number of samples $M$ required satisfies
\begin{equation}
    \E[M]
    = O\left( n\E\left[ \min\left(m, \frac{\log(n/\delta)}{\left[\textnormal{dist}(X,\partial S)\right]^2} \right)\right]\right).
\end{equation}
\end{thm}
\noindent The proof of this Theorem is deferred to \Cref{app:ubProofs}.

Evaluating this expression for different functionals under their corresponding assumptions yields the stated sample complexity upper bounds, as we show in \Cref{app:adaCorr}.

\section{Lower bounds via Wasserstein distance}\label{sec:lbs}
In this section we derive general lower bounds on the sample complexity of functional estimation for both offline and online algorithms, where two different Wasserstein distances play important roles. These Wasserstein-based lower bounds yield tight results for mean and maximum estimation. 

\subsection{General lower bounds based on Wasserstein distance}
A classical technique for proving minimax lower bounds is Le Cam's two-point method \citep{le2000asymptotics}: let $F_1$ and $F_2$ be two distributions with $|g(F_1) - g(F_2)|\ge 2\varepsilon$, and let $p_{\pi, F_1}$ and $p_{\pi, F_2}$ be the probability distributions of all observations queried by policy $\pi$ under the true population distributions $F_1$ and $F_2$, respectively. One version of Le Cam's two-point lower bound \cite[Theorem 2.2]{Tsybakov2009introduction} gives 
\begin{align*}
\inf_{\widehat{g}}\sup_{F\in \{F_1, F_2\}} \mathbb{P}_{F}(|\widehat{g} - g(F)|\ge \varepsilon) \ge \frac{1}{4}\exp\left(-D_{\text{KL}}(p_{\pi,F_1} \| p_{\pi,F_2})\right). 
\end{align*}
Consequently, to construct a lower bound on the PAC sample complexity of estimating $g(F)$, it suffices to find the largest $\varepsilon$ such that there exist $F_1, F_2$ with $|g(F_1) - g(F_2)|\ge 2\varepsilon$ while $D_{\text{KL}}(p_{\pi, F_1}\|p_{\pi, F_2})=O(1)$. 

A key step in the above analysis is to upper bound the KL divergence $D_{\text{KL}}(p_{\pi,F_1}\|p_{\pi,F_2})$, which differs significantly between offline and online algorithms. For offline algorithms, the learner samples $n$ arms i.i.d. from $F$ with average rewards $X_1,\cdots,X_n\sim F$, and each arm is pulled $m$ times with Gaussian observations. Consequently, $p_{\pi, F} = 
= (F*\mcN(0,1/m))^{\otimes n}$, where $p^{\otimes n}$ denotes the $n$-fold product distribution and $*$ denotes the convolution operation. The following lemma presents an upper bound on the KL divergence for offline algorithms. 

\begin{lem}\label{lemma:Wasserstein_2}
For any offline algorithm $\pi$ defined in Section \ref{ssec:offline}, it holds that
\begin{align*}
D_{\text{\rm KL}}(p_{\pi,F_1}\| p_{\pi, F_2}) \le \frac{mn}{2}\mcW_2^2(F_1, F_2),
\end{align*}
where $\mcW_2(P,Q)$ is the Wasserstein-2 distance defined as $\mcW_2^2(P,Q)=\inf_{\gamma\in \Gamma}\mbE_{(X,Y)\sim \Gamma}[(X-Y)^2]$, with $\Gamma$ being the class of all couplings between $P$ and $Q$. 
\end{lem}


For online algorithms the distribution $p_{\pi, F}$ is no longer a product distribution as actions can depend on past observations. As a result, the KL divergence becomes larger, but still enjoys an upper bound based on another Wasserstein distance. 
\begin{lem}\label{lemma:Wasserstein_infty}For any online algorithm $\pi$ which queries $T$ samples, it holds that
\begin{align*}
D_{\text{\rm KL}}(p_{\pi,F_1}\| p_{\pi, F_2}) \le \frac{T}{2}\mcW_\infty^2(F_1, F_2),
\end{align*}
where $\mcW_\infty(P,Q)$ is the Wasserstein-$\infty$ distance: $\mcW_\infty(P,Q)=\inf_{\gamma\in \Gamma}\mathit{\rm esssup}_{(X,Y)\sim\Gamma}|X-Y|$, with $\Gamma$ being the class of all couplings between $P$ and $Q$.
\end{lem}

As $\mcW_2(P,Q)\le \mcW_\infty(P,Q)$, the upper bound of Lemma \ref{lemma:Wasserstein_infty} is no smaller than that of Lemma \ref{lemma:Wasserstein_2}, showing the stronger power of online algorithms. The following corollary is then immediate from Lemmas \ref{lemma:Wasserstein_2} and \ref{lemma:Wasserstein_infty}. 

\begin{cor}\label{corollary:Wasserstein}
The sample complexity of $(\epsilon,.1)$-PAC estimation of $g(F)$ is 
$$
\Omega(1/\min\{\mcW_2^2(F_1,F_2): F_1, F_2\in \mathcal{F}, |g(F_1)-g(F_2)|\ge 2\varepsilon\})
$$
for offline algorithms, and is 
$$
\Omega(1/\min\{\mcW_\infty^2(F_1,F_2): F_1, F_2\in \mathcal{F}, |g(F_1)-g(F_2)|\ge 2\varepsilon\})
$$
for online algorithms.
\end{cor}

In the remainder of this section, we show that Corollary \ref{corollary:Wasserstein} leads to tight lower bounds for mean and maximum estimations for both offline and online settings. 

\subsection{Lower bounds for mean estimation}
Consider two distributions $F_1$ and $F_2$ which are Dirac masses supported on $1/2 - \varepsilon$ and $1/2 + \varepsilon$, respectively. Clearly $\mcW_2(F_1,F_2)=\mcW_\infty(F_1,F_2)=2\varepsilon$, which is the best possible as $\mcW_2(F_1,F_2)\ge |\mathit{\rm mean}(F_1)-\mathit{\rm mean}(F_2)|\ge 2\varepsilon$. Corollary \ref{corollary:Wasserstein} gives the following lower bounds. 
\begin{cor}\label{corollary:lowerbound_mean}
The $(\varepsilon,.1)$-PAC sample complexity for mean estimation is $\Omega(\varepsilon^{-2})$ for both offline and online algorithms. 
\end{cor}

\subsection{Lower bounds for maximum estimation}
For maximum estimation, the Wasserstein distances $\mcW_2$ and $\mcW_\infty$ behave differently, as summarized in the following lemma. Let $\mathcal{F}_\beta$ be the class of densities satisfying Assumption \ref{assum:max}. 
\begin{lem}\label{lemma:lowerbound_maximum}
For $\varepsilon\in (0,1/2)$, it holds that
\begin{align*}
\min\{\mcW_2(F_1, F_2): F_1, F_2 \in \mathcal{F}_\beta, |\mathit{\rm max}(F_1) - \mathit{\rm max}(F_2)|\ge 2\varepsilon\} &= O(\varepsilon^{\beta/2+1}); \\
\min\{\mcW_\infty(F_1, F_2): F_1, F_2 \in \mathcal{F}_\beta, |\mathit{\rm max}(F_1) - \mathit{\rm max}(F_2)|\ge 2\varepsilon\} &= O(\varepsilon); \\
\min\{D_{\text{\rm KL}}(F_1 \| F_2): F_1, F_2 \in \mathcal{F}_\beta, |\mathit{\rm max}(F_1) - \mathit{\rm max}(F_2)|\ge 2\varepsilon\} &= O(\varepsilon^\beta). 
\end{align*}
\end{lem}
Note that we have included another term $D_{\text{KL}}(F_1 \| F_2)$ in Lemma \ref{lemma:lowerbound_maximum} as it can provide a better lower bound than using $\mcW_\infty$ if $\beta \ge 2$, as $D_{\text{\rm KL}}(p_{\pi,F_1}\| p_{\pi, F_2}) \le T\cdot D_{\text{KL}}(F_1\|F_2)$ always holds due to the data-processing inequality (i.e. assuming that all arm rewards are independent). Consequently, we have the following corollary on the sample complexity of maximum estimation.
\begin{cor}\label{corollary:lowerbound_maximum}
The $(\varepsilon, .1)$-PAC sample complexity for maximum estimation over $\mathcal{F}_\beta$ is $\Omega(\varepsilon^{-(\beta+2)})$ for offline algorithms, and $\Omega(\varepsilon^{-\max\{\beta,2\}})$ for online algorithms.
\end{cor}

\section{Lower bounds via thresholding phenomenon}\label{sec:medLB}
Although the Wasserstein distance-based approach in Section \ref{sec:lbs} provides general lower bounds for both offline and online algorithms, and these lower bounds are tight for mean and maximum estimation, sometimes this approach can be loose. For example, Lemma \ref{lemma:lowerbound_maximum} shows that using the $\mcW_\infty$ distance might be looser than using the original KL divergence for maximum estimation. This section provides tight lower bounds for median estimation, revealing a curious thresholding phenomenon. 

\subsection{Thresholding phenomenon for offline algorithms}
Let $\mcF$ denote the set of distributions satisfying Assumption \ref{assum:med}. To use Le Cam's two-point method to prove lower bounds for offline algorithms for median estimation, the key quantity is:  
\begin{align*}
    \mathrm{KL}_\sigma(\varepsilon) := \min\{D_{\text{KL}}(F_1 * \mcN(0,\sigma^2)\| F_2*\mcN(0,\sigma^2) ): F_1, F_2\in \mathcal{F}, |F_1^{-1}(1/2) - F_2^{-1}(1/2)|\ge 2\varepsilon\}.   
\end{align*}
Its inverse $\mathrm{KL}_\sigma^{-1}(\epsilon)$ is 
referred to as the modulus of smoothness of the median with respect to the KL divergence under Gaussian convolution. The Wasserstein-based approach to upper bound $\mathrm{KL}_\sigma(\varepsilon)$ in Lemma \ref{lemma:Wasserstein_2} is the following: let $\mcW_{2,\sigma}(\varepsilon)$ be the counterpart of the above quantity with the KL divergence replaced by the Wasserstein-2 distance, Lemma \ref{lemma:Wasserstein_2} shows that
\begin{align}
\mathrm{KL}_\sigma(\varepsilon) \le \frac{\mcW_{2,\sigma}(\varepsilon)^2}{2\sigma^2} = \Theta\left(\frac{\varepsilon^{2.5}}{\sigma^2}\right), 
\end{align}
an upper bound decreasing continuously with $\sigma$, where the proof of the last identity is presented in the Appendix. However, this upper bound is not tight, as shown in the following lemma. 

\begin{lem}\label{lemma:median_offline}
For $\varepsilon\in (0,1/4)$, $\mathrm{KL}_\sigma(\varepsilon)$ can be characterized as follows: 
\begin{align*}
\mathrm{KL}_\sigma(\varepsilon) \begin{cases}
    \in [C_1\varepsilon^2, C_2\varepsilon^2] &\text{if } \sigma \le c\varepsilon^{1/2}, \\
    \le C(\theta, \kappa) \varepsilon^\kappa &\text{if } \sigma \ge \varepsilon^{1/2-\theta},
\end{cases}
\end{align*}
where $\theta\in (0,1/4), \kappa\in \mathbb{N}$ are arbitrary fixed parameters, and $c, C_1, C_2, C(\theta, \kappa)$ are absolute constants with the last one depending only on $(\theta, \kappa)$. 
\end{lem}

Lemma \ref{lemma:median_offline} shows a thresholding phenomenon as follows: when $\sigma$ increases from $0$ to $1$, the quantity $\mathrm{KL}_\sigma(\varepsilon)$ stabilizes at $\Theta(\varepsilon^2)$ whenever $\sigma \lesssim \varepsilon^{1/2}$; however, when $\sigma$ exceeds this threshold slightly (i.e. $\sigma\gtrsim \varepsilon^{1/2-\theta}$ for any constant $\theta>0$), this quantity immediately drops to $o(\varepsilon^\kappa)$ for every possible $\kappa$. The main intuition behind this thresholding phenomenon is that, if $\sigma = O(\varepsilon^{1/2})$, the ``bandwidth'' of $F_1-F_2$ exceeds that of $\mcN(0,\sigma^2)$, and the convolution is effectively using $\mcN(0,\sigma^2)$ as a Gaussian kernel (which preserves polynomials up to order $2$) 
for smoothing $F_1 - F_2$ (which is second-order differentiable). In contrast, when $\sigma \gg \varepsilon^{1/2}$, the ``bandwidth'' of $F_1-F_2$ could be smaller than $\mcN(0,\sigma^2)$, and the convolution is effectively using $F_1 - F_2$ as a kernel (which could preserve polynomials up to any desired order) for smoothing $\mcN(0,1)$ (which is infinitely differentiable). Approximation theory tells us that the latter approximation error could be much smaller than the former, leading to the thresholding phenomenon. We remark that this phenomenon is not captured by using the $\mcW_2$ distance. 

This thresholding behavior has an important consequence for median estimation. By Lemma \ref{lemma:median_offline} with $\sigma={1/\sqrt{m}}$, PAC learning requires that $m = \Omega(\varepsilon^{2\theta-1})$ 
for any offline algorithm, as otherwise the KL divergence could be made arbitrarily small. When $m$ is large enough, the first line of Lemma \ref{lemma:median_offline} then requires $n=\Omega(\varepsilon^{-2})$ 
to result in a large KL divergence for PAC learning, which comes from the idendity that
$$
\DKL(P^{\otimes n}\|Q^{\otimes n})=n\DKL(P\|Q).
$$
Consequently, we have the following corollary for median estimation using offline algorithms. 

\begin{cor}\label{corollary:lowerbound_median_offline}
Fix any $\theta>0$. The $(\varepsilon,.1)$-PAC sample complexity for median estimation is $\Omega(\varepsilon^{-3+\theta})$ for any offline algorithm.
\end{cor}

\subsection{Thresholding phenomenon for online algorithms}
To prove the PAC lower bound for online algorithms, one first wonders if the same observation in Lemma \ref{lemma:median_offline} could still work. However, a close inspection of the proof reveals an issue: the optimizers $(F_1, F_2)$ in the definition of $\mathrm{KL}_\sigma(\varepsilon)$ are different under the regimes $\sigma = O(\varepsilon^{1/2})$ and $\sigma=\Omega(\varepsilon^{1/2-\theta})$. An online learning algorithm could first identify the right scenario and then choose a proper sample size to tackle the problem, and thus the above lower bound arguments break down. 

To resolve this issue, we aim to choose a proper pair of distributions $(F_1, F_2)$ with $|\mathrm{median}(F_1) - \mathrm{median}(F_2)| \ge 2\varepsilon$, and investigate the behavior of $D_{\text{KL}}(F_1 * \mcN(0,\sigma^2) \| F_2 * \mcN(0,\sigma^2))$ as a function of $\sigma$ \emph{with $(F_1, F_2)$ fixed along the line}. The following lemma shows that, even for some fixed pair $(F_1, F_2)$, a similar thresholding phenomenon still holds for the KL divergence. 

\begin{lem}\label{lemma:median_online}
Fix any $\varepsilon, \theta\in (0,1/4)$, and $\kappa\in \mathbb{N}$. There exist two distributions $F_1, F_2\in \mathcal{F}$ with $|\mathrm{median}(F_1) - \mathrm{median}(F_2)| \ge 2\varepsilon$, and 
\begin{align*}
D_{\text{\rm KL}}(F_1 * \mcN(0,\sigma^2) \| F_2 * \mcN(0,\sigma^2)) \begin{cases}
    \in [C_1\varepsilon^{3/2}, C_2\varepsilon^{3/2}] &\text{if } \sigma \le c\varepsilon^{1/2}, \\
    \le C(\theta, \kappa) \varepsilon^\kappa &\text{if } \sigma \ge \varepsilon^{1/2-\theta},
\end{cases}
\end{align*}
where $c, C_1, C_2, C(\theta, \kappa)$ are absolute constants with the last one depending only on $(\theta, \kappa)$. 
\end{lem}

Compared with Lemma \ref{lemma:median_offline}, Lemma \ref{lemma:median_online} still shows a similar thresholding phenomenon for the KL divergence when $\sigma\gg \varepsilon^{1/2}$, but the KL divergence becomes larger for small $\sigma$ due to the additional constraint that $(F_1, F_2)$ is held fixed. Under the choice of $(F_1, F_2)$ in Lemma \ref{lemma:median_online}, each arm should be pulled at least $\Omega(\varepsilon^{2\theta-1})$ times, while $\Omega(\varepsilon^{\theta-3/2})$ arms need to be pulled in view of the first line. The following theorem makes the above intuition formal. 

\begin{thm}\label{thm:lowerbound_median_online}
The $(\epsilon,.1)$-PAC sample complexity for median estimation is $\Omega(\varepsilon^{-5/2+\theta})$ for any fixed $\theta>0$ and any online algorithm.  
\end{thm}

The formal proof of Theorem \ref{thm:lowerbound_median_online} is more complicated and requires an explicit computation of the KL divergence $D_{\text{KL}}(p_{\pi, F_1} \| p_{\pi, F_2})$. We relegate the full proof to \Cref{app:lbProofsThresh}. This thresholding phenomenon of the noise level also applies to the case of trimmed mean, which is discussed further and an analogous result is proved in Appendix \ref{app:tm}.

\section{Conclusion}\label{sec:conc}
In this work we formulated and studied offline and online algorithms for estimating functionals of distributions. We developed unified algorithms for estimating the mean, median, maximum, and trimmed mean, providing sample complexity upper bounds. We additionally proved information theoretic lower bounds in these settings, which show that our algorithms are optimal up to $\epsilon^c$ where $c$ is a fixed constant arbitrarily close to zero. We used different Wasserstein distances to construct information theoretic lower bounds for mean and maximum estimation, and showed how fundamentally different techniques are required for median and trimmed mean estimation. The lower bounds for median and trimmed mean estimation elucidate an interesting thresholding phenomenon of the noise level to distinguish two distributions after Gaussian convolution, which may be of independent interest. Interesting directions of future work include extending our analysis to non-indicator-based functionals, such as the BH threshold and analyzing the limiting behavior as $\delta\to\infty$.

\section*{Acknowledgements}
Yifei Wang and David Tse were partially supported by NSF Grants CCF-1909499. Tavor Z. Baharav was supported in part by the NSF GRFP and the Alcatel-Lucent Stanford Graduate Fellowship. Yanjun Han is supported by a Simons-Berkeley research fellowship and Norbert Wiener postdoctoral fellowship. Jiantao Jiao was partially supported by NSF Grants IIS-1901252, and CCF-1909499.

\bibliography{reference}

\begin{thebibliography}{}

\bibitem[Aziz et~al., 2018]{aziz2018pure}
Aziz, M., Anderton, J., Kaufmann, E., and Aslam, J. (2018).
\newblock Pure exploration in infinitely-armed bandit models with
  fixed-confidence.
\newblock In {\em Algorithmic Learning Theory}, pages 3--24. PMLR.

\bibitem[Berry et~al., 1997]{berry1997bandit}
Berry, D.~A., Chen, R.~W., Zame, A., Heath, D.~C., and Shepp, L.~A. (1997).
\newblock Bandit problems with infinitely many arms.
\newblock {\em The Annals of Statistics}, 25(5):2103--2116.

\bibitem[Bonald and Proutiere, 2013]{bonald2013two}
Bonald, T. and Proutiere, A. (2013).
\newblock Two-target algorithms for infinite-armed bandits with bernoulli
  rewards.
\newblock {\em Advances in Neural Information Processing Systems}, 26.

\bibitem[Carpentier and Valko, 2015]{carpentier2015simple}
Carpentier, A. and Valko, M. (2015).
\newblock Simple regret for infinitely many armed bandits.
\newblock In {\em International Conference on Machine Learning}, pages
  1133--1141. PMLR.

\bibitem[Chaudhuri and Kalyanakrishnan, 2017]{chaudhuri2017pac}
Chaudhuri, A.~R. and Kalyanakrishnan, S. (2017).
\newblock Pac identification of a bandit arm relative to a reward quantile.
\newblock In {\em Thirty-First AAAI Conference on Artificial Intelligence}.

\bibitem[Chaudhuri and Kalyanakrishnan, 2019]{chaudhuri2019pac}
Chaudhuri, A.~R. and Kalyanakrishnan, S. (2019).
\newblock Pac identification of many good arms in stochastic multi-armed
  bandits.
\newblock In {\em International Conference on Machine Learning}, pages
  991--1000. PMLR.

\bibitem[Cordy and Thomas, 1997]{cordy1997deconvolution}
Cordy, C.~B. and Thomas, D.~R. (1997).
\newblock Deconvolution of a distribution function.
\newblock {\em Journal of the American Statistical Association},
  92(440):1459--1465.

\bibitem[Dattner et~al., 2011]{dattner2011deconvolution}
Dattner, I., Goldenshluger, A., and Juditsky, A. (2011).
\newblock On deconvolution of distribution functions.
\newblock {\em The Annals of Statistics}, pages 2477--2501.

\bibitem[Delaigle et~al., 2008]{delaigle2008deconvolution}
Delaigle, A., Hall, P., and Meister, A. (2008).
\newblock On deconvolution with repeated measurements.
\newblock {\em The Annals of Statistics}, 36(2):665--685.

\bibitem[Hall and Lahiri, 2008]{hall2008estimation}
Hall, P. and Lahiri, S.~N. (2008).
\newblock Estimation of distributions, moments and quantiles in deconvolution
  problems.
\newblock {\em The Annals of Statistics}, 36(5):2110--2134.

\bibitem[Jamieson and Nowak, 2014]{jamieson2014best}
Jamieson, K. and Nowak, R. (2014).
\newblock Best-arm identification algorithms for multi-armed bandits in the
  fixed confidence setting.
\newblock In {\em Information Sciences and Systems (CISS), 2014 48th Annual
  Conference on}, pages 1--6. IEEE.

\bibitem[Kaufmann et~al., 2016]{kaufmann2016complexity}
Kaufmann, E., Capp{\'e}, O., and Garivier, A. (2016).
\newblock On the complexity of best-arm identification in multi-armed bandit
  models.
\newblock {\em The Journal of Machine Learning Research}, 17(1):1--42.

\bibitem[Lai et~al., 1985]{lai1985asymptotically}
Lai, T.~L., Robbins, H., et~al. (1985).
\newblock Asymptotically efficient adaptive allocation rules.
\newblock {\em Advances in applied mathematics}, 6(1):4--22.

\bibitem[Lattimore and Szepesv{\'a}ri, 2020]{lattimore2020bandit}
Lattimore, T. and Szepesv{\'a}ri, C. (2020).
\newblock {\em Bandit algorithms}.
\newblock Cambridge University Press.

\bibitem[Le~Cam et~al., 2000]{le2000asymptotics}
Le~Cam, L., LeCam, L.~M., and Yang, G.~L. (2000).
\newblock {\em Asymptotics in statistics: some basic concepts}.
\newblock Springer Science \& Business Media.

\bibitem[Li and Xia, 2017]{li2017infinitely}
Li, H. and Xia, Y. (2017).
\newblock Infinitely many-armed bandits with budget constraints.
\newblock In {\em Thirty-First AAAI Conference on Artificial Intelligence}.

\bibitem[Rashidinejad et~al., 2021]{rashidinejad2021bridging}
Rashidinejad, P., Zhu, B., Ma, C., Jiao, J., and Russell, S. (2021).
\newblock Bridging offline reinforcement learning and imitation learning: A
  tale of pessimism.
\newblock {\em Advances in Neural Information Processing Systems}, 34.

\bibitem[Schrittwieser et~al., 2021]{schrittwieser2021online}
Schrittwieser, J., Hubert, T., Mandhane, A., Barekatain, M., Antonoglou, I.,
  and Silver, D. (2021).
\newblock Online and offline reinforcement learning by planning with a learned
  model.
\newblock {\em Advances in Neural Information Processing Systems}, 34.

\bibitem[Son and Simon, 2012]{son2012distributed}
Son, L.~K. and Simon, D.~A. (2012).
\newblock Distributed learning: Data, metacognition, and educational
  implications.
\newblock {\em Educational Psychology Review}, 24(3):379--399.

\bibitem[Tsybakov, 2009]{Tsybakov2009introduction}
Tsybakov, A. (2009).
\newblock {\em Introduction to Nonparametric Estimation}.
\newblock Springer-Verlag.

\bibitem[Wang et~al., 2008]{wang2008algorithms}
Wang, Y., Audibert, J.-Y., and Munos, R. (2008).
\newblock Algorithms for infinitely many-armed bandits.
\newblock {\em Advances in Neural Information Processing Systems}, 21.

\bibitem[Wasserman, 2004]{wasserman2004all}
Wasserman, L. (2004).
\newblock {\em All of statistics: a concise course in statistical inference},
  volume~26.
\newblock Springer.

\bibitem[Zhang et~al., 2019]{zhang2019adaptive}
Zhang, M., Zou, J., and Tse, D. (2019).
\newblock Adaptive monte carlo multiple testing via multi-armed bandits.
\newblock In {\em International Conference on Machine Learning}, pages
  7512--7522. PMLR.

\bibitem[Zhang et~al., 2020]{zhang2020determining}
Zhang, M.~J., Ntranos, V., and Tse, D. (2020).
\newblock Determining sequencing depth in a single-cell rna-seq experiment.
\newblock {\em Nature communications}, 11(1):1--11.

\bibitem[Zhang et~al., 2021]{zhang2021reinforcement}
Zhang, Z., Ji, X., and Du, S. (2021).
\newblock Is reinforcement learning more difficult than bandits? a near-optimal
  algorithm escaping the curse of horizon.
\newblock In {\em Conference on Learning Theory}, pages 4528--4531. PMLR.

\end{thebibliography}
\bibliographystyle{apalike}
\clearpage
\section*{Checklist}


\begin{enumerate}

\item For all authors...
\begin{enumerate}
  \item Do the main claims made in the abstract and introduction accurately reflect the paper's contributions and scope?
    \answerYes{Results are stated and properly qualified.}
  \item Did you describe the limitations of your work?
    \answerYes{}
  \item Did you discuss any potential negative societal impacts of your work?
    \answerNA{}
  \item Have you read the ethics review guidelines and ensured that your paper conforms to them?
    \answerYes{}
\end{enumerate}

\item If you are including theoretical results...
\begin{enumerate}
  \item Did you state the full set of assumptions of all theoretical results?
    \answerYes{All theorem's have clearly stated assumptions}
        \item Did you include complete proofs of all theoretical results?
    \answerYes{All proofs are detailed in the Appendix}
\end{enumerate}

\item If you ran experiments...
\begin{enumerate}
  \item Did you include the code, data, and instructions needed to reproduce the main experimental results (either in the supplemental material or as a URL)?
    \answerNA{}
  \item Did you specify all the training details (e.g., data splits, hyperparameters, how they were chosen)?
    \answerNA{}
        \item Did you report error bars (e.g., with respect to the random seed after running experiments multiple times)?
    \answerNA{}
        \item Did you include the total amount of compute and the type of resources used (e.g., type of GPUs, internal cluster, or cloud provider)?
    \answerNA{}
\end{enumerate}

\item If you are using existing assets (e.g., code, data, models) or curating/releasing new assets...
\begin{enumerate}
  \item If your work uses existing assets, did you cite the creators?
    \answerNA{}
  \item Did you mention the license of the assets?
    \answerNA{}
  \item Did you include any new assets either in the supplemental material or as a URL?
    \answerNA{}
  \item Did you discuss whether and how consent was obtained from people whose data you're using/curating?
    \answerNA{}
  \item Did you discuss whether the data you are using/curating contains personally identifiable information or offensive content?
    \answerNA{}
\end{enumerate}

\item If you used crowdsourcing or conducted research with human subjects...
\begin{enumerate}
  \item Did you include the full text of instructions given to participants and screenshots, if applicable?
    \answerNA{}
  \item Did you describe any potential participant risks, with links to Institutional Review Board (IRB) approvals, if applicable?
    \answerNA{}
  \item Did you include the estimated hourly wage paid to participants and the total amount spent on participant compensation?
    \answerNA{}
\end{enumerate}

\end{enumerate}


\clearpage

\appendix 

\section{Extensions of the formulation}\label{app:extension}
The formulation of the functional $g$ can be extended in several different ways. First, we note that we can extend the set $S(F)$ to a finite union of disjoint closed intervals, i.e., $S(F)=\cup_{i=1}^kS_i(F)$, where $S_i(F)$ is a closed interval for $i\in[k]$. This is because $\mbE[X|S(F)]$ can be estimated based on estimations of $\mbE[X|S_i(F)]$ via
\begin{equation}
    \mbE[X|S(F)]=\frac{\mbE[X\mbI(X\in S(F))]}{\P(X\in S(F))}=\frac{\sum_{i=1}^k \P(X\in S_i(F)) \mbE[X|S_i(F)]}{\sum_{i=1}^k \P(X\in S_i(F))}.
\end{equation}
Observe that this definition naturally extends to cases where the distribution is continuous, where the density of $F$ at $X$ can be substituted for $\P(X \in S(F))$ for singleton sets $S(F)$. 
We can also consider a more general class of functionals
\begin{equation}
    g(F) = \E \left[ h(X) | X \in S(F) \right],
\end{equation}
where $h$ is a differentiable function.
However, when we take the limit $\epsilon\to 0$, we see that for any fixed distribution $F$ and fixed function $h$ the reweighting induced by $h$ does not matter.
Assuming that we knew whether $X_i\in S(F)$ for each $i$, we would simply want to sample $X_i \propto h'(X_i) \epsilon^{-r}$ for some $r$.
Since $h$ is differentiable, this is simply reweighting by a constant factor, which does not show up in our $\epsilon$ dependence.
Thus, we can safely only consider the weighting functional $h(x)=x$, which retains the central elimination aspect of this setting (determining whether a point is relevant or not). Loosely speaking, for any differentiable function $h$ and smooth and compactly supported $F$, we have that in the limit as $\epsilon \to 0$ it degenerates to one of these settings.

\section{Results for trimmed mean}\label{app:tm}
In this section, we present our upper and lower bound analysis for trimmed mean via both online and offline sampling algorithms.
\subsection{Upper bound for offline algorithms}
For trimmed mean, the following statements are assumed to hold:
\begin{assum}\label{assum:trim}
There exist constants $c_0,c_1,\dots,c_5$ such that 
\begin{itemize}
\item $\int x^2 dF(x)\leq c_0$.
    \item $F'(x)\geq c_1$ for $|x-F^{-1}(\alpha)|\lesssim \epsilon$ and $|x-F^{-1}(1-\alpha)|\lesssim \epsilon$.
    \item $|F^{(2)}(x)|\leq c_2$ for $|x-F^{-1}(\alpha)|\lesssim \sqrt{\epsilon}$ and $|x-F^{-1}(1-\alpha)|\lesssim \sqrt{\epsilon}$.
    \item $F'(x)\leq c_3$ for $|x-F^{-1}(\alpha)|\lesssim \epsilon$ and $|x-F^{-1}(1-\alpha)|\lesssim \epsilon$.
    \item $\max\{|F^{-1}(\alpha)|,|F^{-1}(1-\alpha)|\}\leq c_4$, $\min\{|F^{-1}(\alpha)|,|F^{-1}(1-\alpha)|\}\geq c_5$.
\end{itemize}
\end{assum}
The first assumption is to ensure that the mean and variance of $F$ is upper bounded, which is slightly stronger than the assumption for mean. The second assumption is to ensure that the $\alpha$ and $1-\alpha$ quantiles of $F$ is well-defined. The third assumption ensures that the distribution has Lipschitz-continuous density around the quantiles. The forth assumption precludes the distributions which have lots of mass around the $\alpha$ and $1-\alpha$ quantiles. The fifth assumption ensures that the $\alpha$ and $1-\alpha$ quantiles are upper-bounded and bounded away from $0$. The following proposition gives the choice of $(n,m)$ to obtain the $(\epsilon,\delta)$-PAC approximation of the trimmed mean.
\begin{prop}\label{prop:trim_ub}
Suppose that Assumption \ref{assum:trim} holds. Then, by choosing $m\geq C_1\epsilon^{-1}\log\epsilon^{-1}$ and $n\geq C_2\epsilon^{-2}\delta^{-1}$, the estimator $G_{n,m}$ is an $(\epsilon,\delta)$-PAC approximation of $g(F)$. Here $C_1,C_2$ are constants which can be expressed by $c_0,\dots,c_5$. Thus, the offline sampling algorithm takes overall $O(\epsilon^{-3}\log(1/\epsilon))$ samples. 
\end{prop}

\subsection{Lower bounds for offline algorithms}
Similar to the analysis for estimating median, we consider the following quantity
\begin{align*}
    \mathrm{KL}_\sigma(\varepsilon) := \min\{D_{\text{KL}}(F_1 * \mcN(0,\sigma^2)\| F_2*\mcN(0,\sigma^2) ): F_1, F_2\in \mathcal{F}, |g(F_1)-g(F_2)|\ge 2\varepsilon\}.   
\end{align*}
Analogously, we have the following bounds on the above quantity with respect to the magnitude of noise $\sigma$.
\begin{lem}\label{lemma:lowerbound_tm_offline}
For $\varepsilon\in (0,1/4)$, the following characterization of $\mathrm{KL}_\sigma(\varepsilon)$ holds as a function of $\sigma$: 
\begin{align*}
\mathrm{KL}_\sigma(\varepsilon) \begin{cases}
    \in [C_1\varepsilon^2, C_2\varepsilon^2] &\text{if } \sigma \le c\varepsilon^{1/2}, \\
    \le C(\theta, \kappa) \varepsilon^\kappa &\text{if } \sigma \ge \varepsilon^{1/2-\theta},
\end{cases}
\end{align*}
where $\theta\in (0,1/4), \kappa\in \mathbb{N}$ are arbitrary parameters, and $c, C_1, C_2, C(\theta, \kappa)$ are absolute constants with the last one depending only on $(\theta, \kappa)$.
\end{lem}
In the same manner, we have the following corollary.
\begin{cor}\label{corollary:lowerbound_tm_offline}
Fix any $\theta>0$. The $(\varepsilon,.1)$-PAC sample complexity for trimmed mean estimation is $\Omega(\varepsilon^{-3+\theta})$ for any offline algorithm.
\end{cor}



\subsection{Lower bounds for online algorithms}
Analogous to the results for median, we start with the following lemma to give bounds of KL divergence between two distributions after the convolution.

\begin{lem}\label{lemma:tm_online}
Fix any $\varepsilon, \theta\in (0,1/4)$, and $\kappa\in \mathbb{N}$. There exists two distributions $F_1, F_2\in \mathcal{F}$ with $|g(F_1) - g(F_2)| \ge 2\varepsilon$, and 
\begin{align*}
D_{\text{\rm KL}}(F_1 * \mcN(0,\sigma^2) \| F_2 * \mcN(0,\sigma^2)) \begin{cases}
    \in[C_1\varepsilon^{3/2}, C_2\varepsilon^{3/2}] &\text{if } \sigma \le c\varepsilon^{1/2}, \\
    \le C(\theta, \kappa) \varepsilon^\kappa &\text{if } \sigma \ge \varepsilon^{1/2-\theta},
\end{cases}
\end{align*}
where $c, C_1, C_2, C(\theta, \kappa)$ are absolute constants with the last one depending only on $(\theta, \kappa)$. 
\end{lem}

We then show the lower bound for trimmed mean via online sampling algorithms.

\begin{thm}\label{thm:lb_ada_tm}
Suppose that $\epsilon>0$. Denote $\mcF$ as the set of distributions satisfying Assumption \ref{assum:trim}. Consider an online algorithm $\pi$ with a fixed budget $t$ which outputs $\hat{G}$. Then, for any $\theta\in(0,1/4)$, there exists at least one distribution $F\in \mcF$, such that
\begin{equation}
    \P(|\hat{G}-g(F)|>\epsilon)\geq \frac{1}{4}\exp\pp{-ct\epsilon^{2.5-2\theta}},
\end{equation}
where $c>0$ is a constant. 
\end{thm}

\section{Proofs of upper bounds for offline algorithms} \label{app:offlinenm}


\subsection{Mean}
Here we present the proof of Proposition \ref{prop:uni_ub_mean}. 
\begin{proof}
Let $X\sim F$ and $Z\sim \mcN(0,1/m)$ are independent random variables. Then, we have
\begin{equation*}
    \mbE[X+Z]=\mbE[X]+\mbE[Z]=\mbE[X].
\end{equation*}
This implies that $g(F_m)=g(F)$ for any $m\geq 1$. Therefore, we can simply take $m=1$. Then, we note that
\begin{equation*}
    \Var[X+Z]=\Var[X]+\Var[Z]\leq c+1.
\end{equation*}
This implies that $\Var_{\hat{X}\sim F_1}[\hat{X}]\leq c+1$. According to the Chebyshev inequality, we have
\begin{equation*}
    \P(|G_{n,m}-g(F_m)|\geq \epsilon)\leq \frac{\Var_{\hat{X}\sim F_1}[\hat{X}]}{n\epsilon^2}\leq \frac{c+1}{n\epsilon^2}.
\end{equation*}
Therefore, by taking $n\geq \delta^{-1}(c+1)\epsilon^{-2}$, we have
\begin{equation*}
    \P(|G_{n,m}-g(F_m)|\leq \epsilon)\geq 1- \delta.
\end{equation*}
Hence, it takes $mn=O(\epsilon^{-2})$ samples to provide an $(\epsilon,\delta)$-PAC approximation of $g(F)$.
\end{proof}
\subsection{Median}
Consider the following conditions
\begin{enumerate}
    \item [(A1)] For $x\in \mbR$, there exists $c_1, t_1>0$ such that for all $t$ satisfying $0\leq |t-x|\leq t_1$, $F'(t)\geq c_1$.
    \item [(A2)] For $x\in\mbR$, there exists $c_2, t_2>0$ such that 
    $$
    |F'(x_1)-F'(x_2)||\leq c_2.
    $$
    for $x_1,x_2\in[x-t_2,x+t_2]$.
\end{enumerate}
We can view Assumption \ref{assum:med} as follows. Let $\eta=g(F)=F^{-1}(0.5)$. $F$ satisfies (A1) with $(\eta, c_1,t_1)$ and satisfies (A2) with $(\eta,c_2,t_2)$ while $t_1\gtrsim \epsilon$ and $t_2\gtrsim \sqrt{\epsilon}$. Denote $\rho(x)=F'(x)$. Let $\rho_m=\rho*\varphi_{1/m}$ as the pdf of the distribution of $\hat X_i$. We start with Lemma \ref{lem:a1} to show that under suitable choice of $m$, $F_m$ also satisfies (A1).

\begin{lem}\label{lem:a1}
Let $\eta=g(F)$. Assume that $F$ satisfies (A1) with $(\eta, c_1,t_1)$. Suppose that $m^{-1/2}\leq t_1/2$. Then, $F_{m}$ satisfies (A1) with $(\eta, c_1/4,t_1)$.
\end{lem}
\begin{proof}
It is sufficient to show that for $x\in[\eta-t_1,\eta+t_1]$, $\rho_{m}(x)>c_1/4$. As $m^{-1/2}<t_1/2$, 
\begin{equation*}
    \int_{0}^{t_1} \varphi_{1/m}(x)dx\geq \int_{0}^{2m^{-1/2}} \varphi_{1/m}(x)dx=\int_0^2\varphi_1(x)dx\geq 1/4.
\end{equation*}
Therefore, for $x\in[\eta,\eta+t_1]$, we have
\begin{equation*}
    \rho_{m}(x) = \int_z \varphi_{1/m}(z)\rho(x-z)dz\geq c_1\int_{0}^{t_1} \varphi_{1/m}(z) dz\geq c_1/4.
\end{equation*}
Similarly, for $x\in[\eta-t_1,\eta]$, we have
\begin{equation*}
    \rho_{m}(x) = \int_z \varphi_{1/m}(z)\rho(x-z)dz\geq c_1\int_{-t_1}^{0} \varphi_{1/m}(z) dz=c_1\int_0^{t_1}\varphi_{1/m}(z) dz\geq c_1/4.
\end{equation*}
This completes the proof. 
\end{proof}
To prove Proposition \ref{prop:med_ub}, we introduce the following proposition to give a point-wise bound on the difference between $F(x)$ and $F_m(x)$.
\begin{prop}\label{prop:cdf_diff}
Suppose that $F$ satisfies (A2) with $(x, c_2,t_2)$ and $\sqrt{4\log(2m^{-1/2})}m^{-1/2}\leq t_2$. Then, we have
\begin{equation*}
|F_m(x)-F(x)|\leq\frac{c_2+1}{2}m^{-1}.
\end{equation*}
\end{prop}
\begin{proof}
With $k=\sqrt{4\log(2m^{-1/2})}$, we have
\begin{equation*}
\int^{-km^{-1/2}}_{-\infty} \varphi_{1/m}(y)dy=\int_{km^{-1/2}}^\infty \varphi_{1/m}(y)dy\leq e^{-k^2/2}\leq \frac{1}{4}m^{-1}.
\end{equation*}
For $|y|\leq t_2$, as $|F^{(2)}(x-y)|\leq c_2$, it follows that
\begin{equation*}
|F(x-y)-F(x)-y\rho(x)|\leq \frac{c_2y^2}{2}.
\end{equation*}
Note that $km^{-1/2}=\sqrt{4\log(2m^{-1/2})}m^{-1/2}\leq t_2$, we have
\begin{equation*}
\begin{aligned}
&|F_m(x)-F(x)|\\
=&\left|\int (F(x-y)-F(x)) \varphi_{1/m}(y)dy\right|\\
 \leq &\int^{-km^{-1/2}}_{-\infty}  \varphi_{1/m}(y)|F(x-y)-F(x)|dy+\int_{km^{-1/2}}^\infty  \varphi_{1/m}(y)|F(x-y)-F(x)|dy\\
 &+\left| \int_{-km^{-1/2}}^{km^{-1/2}} (F(x-y)-F(x)-y\rho(t))  \varphi_{1/m}(y)dy\right|+\left| \int_{-km^{-1/2}}^{km^{-1/2}} y\rho(t)  \varphi_{1/m}(y) dy\right|\\
 \leq &\frac{m^{-1}}{2}+\int_{-km^{-1/2}}^{km^{-1/2}} |F(x-y)-F(x)-y\rho(x)|  \varphi_{1/m}(y)dy\\
\leq&\frac{m^{-1}}{2}+\frac{c_2}{2}\int_{-km^{-1/2}}^{km^{-1/2}}y^2\varphi_{1/m}(y)dy\leq\frac{c_2+1}{2}m^{-1}.
\end{aligned}
\end{equation*}
This completes the proof.
\end{proof}
We restate Proposition \ref{prop:med_ub} as follows and present the proof. 
\begin{prop}
Suppose that $\delta\in(0,1)$. Assume that (A1) holds at $\eta$ with $(c_1,t_1)$ and (A2) holds at $\eta$ with $(c_2,t_2)$. Suppose that $t_1\geq\epsilon/2$ and $t_2\gtrsim \sqrt{\epsilon}$. 
Then, with $m\geq \frac{4(c_2+1)\epsilon^{-1}}{c_1}$ and $n\geq \frac{28\epsilon^{-2}\log\delta^{-1}}{c_1^2}$, $G_{n,m}$ is an $(\epsilon,\delta)$-PAC approximation of $g(F)$.
\end{prop}
\begin{proof}
Suppose that we use $n$ points and $m$ samples per point. From our choice of $m$, we have 
$$
\sqrt{4\log(2m^{1/2})}m^{-1/2}\leq t_2,\frac{c_2+1}{2}m^{-1}\leq \frac{c_1\epsilon}{8}.
$$
Let $\eta_m=g(F_m)$ and $\eta=g(F)$. From Proposition \ref{prop:cdf_diff}, we have
\begin{equation*}
    |F_m(\eta_m)-F_m(\eta)| = |F(\eta)-F_m(\eta)|\leq c_1\epsilon/8.
\end{equation*}
From Lemma \ref{lem:a1}, we note that $F_{m}$ satisfies (A1) with $(\eta,c_1/4,t_1)$. If $|\eta_\sigma-\eta|\geq t_1$, then, we have
\begin{equation*}
\begin{aligned}
    |F_m(\eta_\sigma)-F_m(\eta)|
    \geq \min\{ |F_m(\eta+t_1)-F_m(\eta)|,|F_m(\eta-t_1)-F_m(\eta)|\}
    \geq \frac{c_1t_1}{4},
\end{aligned}
\end{equation*}
which leads to a contradiction. Therefore, we have
\begin{equation*}
    c_1\epsilon/8\geq  |F_m(\eta_\sigma)-F_m(\eta)|\geq c_1|\eta_\sigma-\eta|/4.
\end{equation*}
This implies that $|\eta-\eta_m|\leq \epsilon/2$. As $\epsilon\leq t_1/2$, we note that $F_{m}$ satisfies (A1) with $(\eta_m,c_1/4,t_1/2)$. From the choice of $n$, according to Lemma \ref{lem:quant_ub}, we have
\begin{equation*}
   \P( |G_{n,m}-\eta_m|\leq \epsilon/2)\geq 1-\delta.
\end{equation*}
Under the event $|G_{n,m}-\eta_m|\leq \epsilon/2$, we have
\begin{equation*}
    |G_{n,m}-\eta|\leq  |G_{n,m}-\eta_m|+|\eta_m-\eta|\leq \epsilon.
\end{equation*}
This completes the proof. 
\end{proof}

\subsection{Maximum}
In this case, the estimator for the noiseless samples writes $G_n=\max_{i\in[n]}X_n$. We first show that for sufficiently large $n$, $F(G_n)$ can be close to $1$. 
\begin{prop}\label{prop:max_rel}
Suppose that $\epsilon>0$ and $\delta\in(0,1)$. Then, for $n\geq \epsilon^{-1}\log(2/\delta)$, we have  $\P(F(G_n)\geq 1-\epsilon)\geq 1-\delta/2$.
\end{prop}
\begin{proof}
Consider a fixed number of points $n$. Note that $G_n=\max_{i\in[n]}X_i$. Therefore, we have
\begin{equation*}
\begin{aligned}
    &\P(F(G_n)\leq 1-\epsilon)= \P(F(X_i)\leq 1-\epsilon,\forall i\in[n])\\
    =&(1-\epsilon)^n \leq \exp(\epsilon^{-1}\log(2/\delta)\log(1-\epsilon))\leq \delta/2.
\end{aligned}
\end{equation*}
Here we utilize that $\log(1-\epsilon)\leq -\epsilon$. This completes the proof. 
\end{proof}
Then, based on the $\beta$-regularity of $F$, we show that $G_n$ can be close to $g(F)$ when $n$ is large.
\begin{prop}\label{prop:max_abs}
Let $\epsilon>0$ and $\delta\in(0,1)$. Denote $\eta=g(F)$. Suppose that Assumption \ref{assum:max} holds.  Then, with $n=c_1^{-1}\epsilon^{-\beta}\log(2/\delta)$ points, we have $\P(|G_n-\eta|\leq \epsilon)\geq 1-\delta/2$.
\end{prop}
\begin{proof}
From Proposition \ref{prop:max_rel}, we note that 
\begin{equation*}
    \P(F(G_n)\geq 1-c_1\epsilon^\beta)\geq 1-\delta/2.
\end{equation*}
According to Assumption \ref{assum:max}, $F(G_n)\geq 1-c_1\epsilon^\beta$ implies that $G_n\geq \eta-\epsilon$. As $G_n=\max_{i\in[n]}X_i\leq \eta$, this completes the proof. 
\end{proof}

We first choose $n\geq c_1^{-1}(\epsilon/2)^{-\beta}\log(2/\delta)$. From Proposition \ref{prop:max_abs}, this guarantees that $\P(|G_n-\eta|\leq \epsilon/2)\geq 1-\delta/2$. Then, by choosing $m\geq 4\epsilon^{-2}\log(2n\delta^{-1})$, we have
\begin{equation*}
    \P(|X_i-\hat{X}_i|\leq \epsilon/2)\geq 1-e^{m^{-1}\epsilon^{-2}}\geq 1-\delta/(2n).
\end{equation*}
Here we utilize the tail bound of Gaussian distributions and the fact that $X_i-\hat{X}_i\sim\mcN(0,1/m)$. As $G_n=\max_{i\in[n]}X_i$ and $G_{n,m}=\max_{i\in[n]}\hat{X}_i$, conditioned on $\{|X_i-\hat{X}_i|\leq \epsilon/2,\forall i\in[n]\}$, we have $|G_n-G_{n,m}|\leq \epsilon/2$. Therefore, it follows that 
\begin{equation*}
    \P(|G_n-G_{n,m}|\leq \epsilon/2)\geq P\pp{|X_i-\hat{X}_i|\leq \epsilon/2,\forall i\in[n]}\geq 1-n\delta/(2n)=1-\delta/2.
\end{equation*}
In summary, we have $\P(|\eta-G_{n,m}|\leq \epsilon)\geq \P(|G_n-G_{n,m}|\leq \epsilon/2)+\P(|G_n-\eta|\leq \epsilon/2)-1\geq 1-\delta$. 

\subsection{Trimmed mean}
Consider the following conditions
\begin{enumerate}
\item [(B1)] There exists constant $c_0>0$ such that $\int_{-\infty}^{\infty} x^2 dF(x)\leq c_0$.
\item [(B2)] For $x\in \mbR$, there exists $c_1, t_1>0$ such that for all $t$ satisfying $0\leq |t-x|\leq t_1$, $ F'(t)\geq c_1$.
    \item [(B3)] For $x\in\mbR$, there exists $c_2, t_2>0$ such that for $x_1,x_2\in[x-t_2,x+t_2]$,
    $$
    |F'(x_1)-F'(x_2)|\leq c_2.
    $$
    \item [(B4)] For $x\in \mbR$, there exists $c_3, t_3>0$ such that for all $t$ satisfying $0\leq |t-x|\leq t_3$, $ F'(t)\leq c_3$.
    \item [(B5)] There exists constants $c_4,c_5>0$ such that $\max\{|F^{-1}(\alpha)|,|F^{-1}(1-\alpha)|\}\leq c_4$, $\min\{|F^{-1}(\alpha)|,|F^{-1}(1-\alpha)|\}\geq c_5$.
\end{enumerate}
We can view Assumption \ref{assum:trim} as follows. $F$ satisfies (B1) with $c_0$ and (B5) with $c_4,c_5$. At $F^{-1}(\alpha)$ and $F^{-1}(1-\alpha)$, $F$ satisfies (B2) with $(c_1,t_1)$, satisfies (B3) with $(c_2,t_2)$ and satisfies (B4) with $(c_3,t_3)$. Here $t_1,t_3\gtrsim \epsilon$ and $t_2\gtrsim \sqrt{\epsilon}$.

We first show that for $n\geq O(\epsilon^{-2})$, the empirical estimator of the trimmed mean from noiseless samples will be close to the trimmed mean.

\begin{prop}\label{prop:tm_noiseless}
Suppose that Assumption \ref{assum:trim} holds. Let $\delta\in(0,1)$. Suppose that $\epsilon>0$ is sufficiently small. For 
$$
n\geq (2\alpha-1)^{-2}(4c_3c_4+1)^{-2}\epsilon^{-2}\max\{16c_0\delta^{-1},4\log(4/\delta)c_1^{-2}\},
$$
with probability at least $1-\delta$, we have 
\begin{equation*}
    \left|\frac{1}{n}\sum_{i=\flr{\alpha n}}^{\flr{(1-\alpha)n}} X_{(i)}-\int_{F^{-1}(\alpha)}^{F^{-1}(1-\alpha)}xdF(x)\right|\leq \epsilon.
\end{equation*}
\end{prop}

We defer the proof to Appendix \ref{proof_prop:tm_noiseless}. Then, we prove for the noisy case. 

\begin{lem}\label{lem:tm_noise}
Suppose that Assumption \ref{assum:trim} holds. Then, we have
\begin{equation*}
\begin{aligned}
&\left|\int_{F^{-1}(\alpha)}^{F^{-1}(1-\alpha)} xdF(x)-\int_{F^{-1}(\alpha)}^{F^{-1}(1-\alpha)} x dF_{m}(x)\right|\\
\leq& m^{-1}\pp{4\sqrt{c_0}+k^2(2c_2c_4+2c_3)+4+2\sqrt{c_0}c_2k^2m^{-1/2}} =O(m^{-1}\log(m)),
\end{aligned}
\end{equation*}
where $k=\sqrt{2\log\pp{m/2}}$.
\end{lem}


We leave the proof in Appendix \ref{proof_lem:tm_noise}. From the median proof, analogously, we also have
\begin{equation*}
    |F^{-1}(\alpha)-F_{m}^{-1}(\alpha)|\leq O (m^{-1} \log(m)) \quad,\quad |F^{-1}(1-\alpha)-F_{m}^{-1}(1-\alpha)|\leq O(m^{-1}\log(m)).
\end{equation*}
Then, we have the bound 
\begin{equation*}
\begin{aligned}
      & \left|\int_{F_m^{-1}(\alpha)}^{F_m^{-1}(1-\alpha)}xdF_m(x)-\int_{F^{-1}(\alpha)}^{F^{-1}(1-\alpha)}xdF(x)\right|\\
    \leq&\left|\int_{F^{-1}(\alpha)}^{F^{-1}(1-\alpha)}xdF_m(x)-\int_{F^{-1}(\alpha)}^{F^{-1}(1-\alpha)}xdF(x)\right|\\
    &+\left|\int_{F^{-1}(\alpha)}^{F^{-1}(1-\alpha)}xdF_m(x)-\int_{F_m^{-1}(\alpha)}^{F_m^{-1}(1-\alpha)}xdF_m(x)\right|\\
    \leq & O(m^{-1} \log(m)).
\end{aligned}
\end{equation*}
Here we utilize that $|x|F'_m(x)$ is upper bounded. Therefore, by choosing $m=O(\epsilon^{-1}\log(1/\epsilon))$, we have 
\begin{equation*}
    \left|\int_{F_m^{-1}(\alpha)}^{F_m^{-1}(1-\alpha)}xdF_m(x)-\int_{F^{-1}(\alpha)}^{F^{-1}(1-\alpha)}xdF(x)\right|\leq \epsilon/2.
\end{equation*}
By choosing $\epsilon$ sufficiently small, $F_m$ also satisfies Assumption \ref{assum:trim} with constants $(2c_0,c_1/2,2c_2,2c_3,2c_4,c_5/2)$. Therefore, with $n\geq O(\epsilon^{-2}\delta^{-1})$, we have
\begin{equation*}
    \P\pp{\left|\int_{F_m^{-1}(\alpha)}^{F_m^{-1}(1-\alpha)}xdF_m(x)-G_{m,n}\right|\leq \epsilon/2}\geq 1-\delta.
\end{equation*}
This completes the proof.

\subsubsection{Proof of Proposition \ref{prop:tm_noiseless}}\label{proof_prop:tm_noiseless}
\begin{proof}
For $\xi>0$, denote the event
\begin{equation*}
    E(\xi)=\{|X_{(\flr{\alpha n})}-F^{-1}(\alpha)|\leq \xi , |X_{(\flr{(1-\alpha) n})}-F^{-1}(1-\alpha)|\leq \xi\}.
\end{equation*}
Choose $\xi<c_5$. From Lemma \ref{lem:quant_ub_dist}, with $n\geq 4\log(4/\delta)c_1^{-2}\xi^{-2}$, we have $\P(E(\xi))\geq 1-\delta/2$. Conditioned on $E(\xi)$, we note that 
\begin{equation*}\label{equ:tm_b1}
    \frac{1}{n}\sum_{i=\flr{\alpha n}}^{\flr{(1-\alpha)n}} X_{(i)}
    \geq \frac{1}{n}\sum_{i\in[n]} X_i\mbI(X_i\in[F^{-1}(\alpha)+\sign(F^{-1}(\alpha))\xi,F^{-1}(1-\alpha)-\sign(F^{-1}(1-\alpha))\xi]),
\end{equation*}
and
\begin{equation*}\label{equ:tm_b2}
    \frac{1}{n}\sum_{i=\flr{\alpha n}}^{\flr{(1-\alpha)n}} X_{(i)}
    \leq \frac{1}{n}\sum_{i\in[n]} X_i\mbI(X_i\in[F^{-1}(\alpha)-\sign(F^{-1}(\alpha))\xi,F^{-1}(1-\alpha)+\sign(F^{-1}(1-\alpha))\xi]).
\end{equation*}
We introduce the following lemma to show the convergence of trimmed mean.
\begin{lem}\label{lem:trunc_mean}
Let $a<b$. 
Then, with $n\geq c_0\epsilon^{-2}\delta^{-1}$, we have 
$$
\P\pp{\left|\int_a^b x dF(x)-\frac{1}{n}\sum_{i=1}^n X_i\mbI(X_i\in[a,b])\right|\leq \epsilon}\geq 1-\delta.
$$
\end{lem}
\begin{proof}
Denote $Y_i=X_i\mbI(X_i\in[a,b])$. Then, $\mbE[Y_i]=\int_a^b x dF(x)$. Note that $\Var[Y_i]\leq \int_a^b x^2 dF(x)\leq \int_{-\infty}^{\infty}x^2dF(x)\leq c_0$. Therefore, from the Chebyshev inequality, we have
\begin{equation*}
\begin{aligned}
&P\pp{\left|\int_a^b x dF(x)-\frac{1}{n}\sum_{i=1}^n X_i\mbI(X_i\in[a,b])\right|\geq \epsilon}\leq \frac{c_0}{n\epsilon^2}\leq \delta.
\end{aligned}
\end{equation*}
This completes the proof. 
\end{proof}

From Lemma \ref{lem:trunc_mean}, for $n\geq 16c_0\xi^{-2}\delta^{-1}$, with probability at least $1-\delta/4$, we have
\begin{equation*}
\begin{aligned}
        &\frac{1}{n}\sum_{i\in[n]} X_i\mbI(X_i\in[F^{-1}(\alpha)+\sign(F^{-1}(\alpha))\xi,F^{-1}(1-\alpha)-\sign(F^{-1}(1-\alpha))\xi])\\
    \geq& \int_{F^{-1}(\alpha)+\sign(F^{-1}(\alpha)}^{F^{-1}(1-\alpha)-\sign(F^{-1}(\alpha)} xdF(x)-\xi
    \geq \int_{F^{-1}(\alpha)}^{F^{-1}(1-\alpha)}xdF(x)-(4c_3c_4+1)\xi,
\end{aligned}
\end{equation*}
and
\begin{equation*}
\begin{aligned}
        &\frac{1}{n}\sum_{i\in[n]} X_i\mbI(X_i\in[F^{-1}(\alpha)-\sign(F^{-1}(\alpha))\xi,F^{-1}(1-\alpha)+\sign(F^{-1}(1-\alpha))\xi])\\
    \leq& \int_{F^{-1}(\alpha)-\sign(F^{-1}(\alpha)}^{F^{-1}(1-\alpha)+\sign(F^{-1}(\alpha)} xdF(x)+\xi
    \leq \int_{F^{-1}(\alpha)}^{F^{-1}(1-\alpha)}xdF(x)+(4c_3c_4+1)\xi.
\end{aligned}
\end{equation*}
Here we utilize that $\xi\leq c_5\leq c_4$ and $|xF'(x)|\leq 2c_3c_4$ around $F^{-1}(\alpha)$ and $F^{-1}(1-\alpha)$. 
Combining the above bound with \eqref{equ:tm_b1} and \eqref{equ:tm_b2}, with probability at least $1-3\delta/4$, we have
\begin{equation*}
     \left|\frac{1}{n}\sum_{i=\flr{\alpha n}}^{\flr{(1-\alpha)n}} X_{(i)}-\int_{F^{-1}(\alpha)}^{F^{-1}(1-\alpha)}xdF(x)\right|\leq (4c_3c_5+1)\xi.
\end{equation*}
Therefore, by letting $\xi=\frac{1}{(2\alpha-1)(4c_3c_5+1)}\epsilon$, we complete the proof. \end{proof}

\subsubsection{Proof of Lemma \ref{lem:tm_noise}}\label{proof_lem:tm_noise}
\begin{proof}
Denote $\sigma=1/\sqrt{m}$. We also denote $F_{\sigma}=:F_m$.  According to the Cauchy-Schwartz inequality, we have
\begin{equation*}
    \pp{\int |x| dF(x)}^2\leq \pp{\int x^2 dF(x) }\pp{\int 1dF(x)}\leq c_0,
\end{equation*}
which implies that $\int |x|dF(x)\leq \sqrt{c_0}$. Let $k>0$ be a constant. Note that
\begin{equation*}
    \int_{k\sigma}^\infty \varphi_{\sigma^2}(x)dx = \int_k^\infty \varphi_{1}(x)dx, \int_{k\sigma}^\infty x\varphi_{\sigma^2}(x)dx=\sigma \int_k^\infty x\varphi_{1}(x)dx.
\end{equation*}
It follows that
\begin{equation*}
\begin{aligned}
     & \int_k^\infty x\varphi_{1}(x)dx =\frac{1}{\sqrt{2\pi}}\int_k^\infty e^{-\frac{x^2}{2}}d\frac{x^2}{2}=\frac{1}{\sqrt{2\pi}}e^{-k^2/2}.
    \end{aligned}
\end{equation*}
We note that $\int_k^\infty \varphi_{1}(x)dx\leq e^{-k^2/2}$. 
By taking $k=\sqrt{-2\log\pp{\sigma^2/2}}$, then, we have
\begin{equation*}
    \int_{k\sigma}^\infty \varphi_{\sigma^2}(x)dx\leq e^{-k^2/2} \leq \frac{1}{2}\sigma^2, \int_{k\sigma}^\infty x\varphi_{\sigma^2}(x)dx\leq \sigma e^{-k^2/2} \leq \frac{1}{2}\sigma^2.
\end{equation*}
We can compute that

\begin{align}
     &\left|\int_a^b x dF_{\sigma}(x) - \int_{a}^{b} y F'(y)dy \right|\notag\\ 
    =&\left|\int_{a\leq y+z\leq b}(y+z) F'(y)\varphi_{\sigma^2}(z)dydz - \int_{a}^{b} y F'(y)dy\right|\notag \\
    \leq &\left|\int_{a\leq y+z\leq b}y F'(y)\varphi_{\sigma^2}(z)dydz- \int_{a}^{b} y F'(y)dy\right|+\left|\int_{a\leq y+z\leq b}zF'(y)\varphi_{\sigma^2}(z)dydz\right|.\notag
\end{align}

In the following two lemmas, we show that both terms in the last line are upper bounded by $O(\sigma^2)$. 
\begin{lem}\label{lem:bnd_1}
We have the bound
\[
    \left|\int_{a\leq y+z\leq b}y F'(y)\varphi_{\sigma^2}(z)dydz- \int_{a}^{b} y F'(y)dy\right|\leq 4\sqrt{c_0}\sigma^2 +k^2\sigma^2((|b|+|a|)c_2+2c_4).\label{term1}
\]
\end{lem}
\begin{lem}\label{lem:bnd_2}
We have the bound
\[
    \left|\int_{a\leq y+z\leq b}zF'(y)\varphi_{\sigma^2}(z)dydz\right|\leq 4\sigma^2+2\sqrt{c_0}c_2k^2\sigma^3.\label{term2}
    \]
\end{lem}
The proofs are left in Appendix \ref{proof_lem:bnd_1} and \ref{proof_lem:bnd_2}.  In summary, we have the bound
\begin{equation*}
\begin{aligned}
    &\left|\int_a^b xdF(x)-\int_a^b x dF_{\sigma}(x)\right|\\
\leq &4\sqrt{c_0}\sigma^2 +k^2\sigma^2((|b|+|a|)c_2+2c_4)+4\sigma^2+2\sqrt{c_0}c_2k^2\sigma^3\\
=&\sigma^2\pp{4\sqrt{c_0}+k^2((|b|+|a|)c_2+2c_4)+4+2\sqrt{c_0}c_2k^2\sigma}
\end{aligned}
\end{equation*}
This completes the proof. 
\end{proof}

\subsubsection{Proof of Lemma \ref{lem:bnd_1}}\label{proof_lem:bnd_1}
We first upper-bound the LHS in \eqref{term1} by the following parts:
\begin{align}
     &\left|\int_{a\leq y+z\leq b}y F'(y)\varphi_{\sigma^2}(z)dydz-\int_{a}^{b} y F'(y)dy \right| \notag \\
    \leq &\left|\int_{a+k\sigma}^{b-k\sigma} \pp{\int_{a-y}^{b-y}\varphi_{\sigma^2}(z)dz}y F'(y)dzdy-\int_{a+k\sigma}^{b-k\sigma}y F'(y)dy \right|\label{term1_1}\\
    &+\left|\pp{\int_{-\infty}^{a-k\sigma}+\int_{b+k\sigma}^\infty}\pp{\int_{a-y}^{b-y}\varphi_{\sigma^2}(z)dz} y F'(y)dy\right|\label{term1_2}\\
    &+\left|\int_{a-k\sigma}^{a+k\sigma} \pp{\int_{a-y}^{b-y}\varphi_{\sigma^2}(z)dz}y F'(y)dy-\int_a^{a+k\sigma} y F'(y)dy\right|\label{term1_4}\\
    &+\left|\int_{b-k\sigma}^{b+k\sigma} \pp{\int_{a-y}^{b-y}\varphi_{\sigma^2}(z)dz}y F'(y)dy-\int_{b-k\sigma}^b y F'(y)dy\right|.\label{term1_5}
\end{align}
For the term \eqref{term1_1}, as $y\in[a+k\sigma,b-k\sigma]$, we have $[-k\sigma,k\sigma]\subseteq [a-y,b-y]$, which implies that
\begin{equation*}
    \left| \int_{a-y}^{b-y}\varphi_{\sigma^2}(z)dz-1\right|\leq \pp{\int_{-\infty}^{-k\sigma}+\int_{k\sigma}^{\infty}}\varphi_{\sigma^2}(z)dz\leq \sigma^2.
\end{equation*}
Hence, we have
\begin{equation*}
\begin{aligned}
     &\left|\int_{a+k\sigma}^{b-k\sigma} \pp{\int_{a-y}^{b-y}\varphi_{\sigma^2}(z)dz}y F'(y)dzdy-\int_{a+k\sigma}^{b-k\sigma} y F'(y)dy\right|\\
    \leq &\sigma^2\int_{a+k\sigma}^{b-k\sigma}|y|F'(y)dy \leq \sigma^2\int_{-\infty}^{\infty} |y|F'(y)dy \leq \sqrt{c_0}\sigma^2.
\end{aligned}
\end{equation*}

For the term \eqref{term1_2}, we note that
\begin{equation*}
\begin{aligned}
    &\left|\pp{\int_{-\infty}^{a-k\sigma}+\int_{b+k\sigma}^\infty}\pp{\int_{y-a}^{y-b}\varphi_{\sigma^2}(z)dz} y F'(y)dy\right|\\
    \leq &\frac{1}{2}\sigma^2\pp{\int_{-\infty}^{a-k\sigma}+\int_{b+k\sigma}^\infty}|y|F'(y)dy\leq \frac{\sqrt{c_0}}{2}\sigma^2.
\end{aligned}
\end{equation*}
Here we utilize that for $y\geq b+k\sigma$ or $y\leq a-k\sigma$, we have
\begin{equation*}
    \int_{a-y}^{b-y}\varphi_{\sigma^2}(z)dz\leq \int_{k\sigma}^\infty \varphi_{\sigma^2}(z)dz\leq \frac{1}{2}\sigma^2.
\end{equation*}
For the term \eqref{term1_4}, we note that
\begin{equation*}
\begin{aligned}
         &\left|\int_{a-k\sigma}^{a+k\sigma}y F'(y)\pp{\int_{a-y}^{b-y}\varphi_{\sigma^2}(z)}dzdy-\int_{a}^{a+k\sigma}yF'(y)dy\right|\notag \\
    \leq &\left|\int_0^{k\sigma}(a+t) F'(a+t)\pp{\int_{-(b-a)-t}^{b-a-t}\varphi_{\sigma^2}(z)dz}dt-\int_{a}^{a+k\sigma}yF'(y)dy\right|
    \\
    &+\left|\int_0^{k\sigma}\pp{(a-t) F'(a-t)-(a+t)F'(a+t)}\pp{\int_{t}^{b-a+t}\varphi_{\sigma^2}(z)}dzdt\right|
    \\
    \leq &\left|\int_0^{k\sigma}(a+t) F'(a+t)\pp{\int_{-(b-a)-t}^{b-a-t}\varphi_{\sigma^2}(z)dz}dt-\int_{0}^{k\sigma}(a+t)F'(a+t)dt\right|\\
    &+\int_0^{k\sigma}|a|F'(a-t)-F'(a+t)|+t|F'(a-t)+F'(a+t)|dt\\
    \leq& \sigma^2 \int_0^{k\sigma}|a+t|F'(a+t)dt +\int_0^{k\sigma}\pp{2|a|c_2t+2c_4t}dt\\
    \leq &\sqrt{c_0}\sigma^2 +k^2\sigma^2(|a|c_2+c_4).
\end{aligned}
\end{equation*}

Similarly, for the term \eqref{term1_5}, we have the bound
\begin{equation*}
\begin{aligned}
\left|\int_{b-k\sigma}^{b+k\sigma}y F'(y)\pp{\int_{a-y}^{b-y}\varphi_{\sigma^2}(z)}dzdy-\int_{b}^{b+k\sigma}yF'(y)dy\right|
 \leq \sqrt{c_0}\sigma^2 +k^2\sigma^2(|b|c_2+c_4).
\end{aligned}
\end{equation*}
\subsubsection{Proof of Lemma \ref{lem:bnd_2}}\label{proof_lem:bnd_2}
For the LHS in \eqref{term2}, we can decompose it into
\begin{align}
     &\left|\int_{a\leq y+z\leq b}z F'(y)\varphi_{\sigma^2}(z)dydz\right|\notag\\
    \leq &\left|\int_{a+k\sigma}^{b-k\sigma}  F'(y)\pp{\int_{a-y}^{b-y}z\varphi_{\sigma^2}(z)}dzdy\right|\label{term2_1}\\
    &+\left|\pp{\int_{-\infty}^{a-k\sigma}+\int_{b+k\sigma}^\infty}\pp{\int_{a-y}^{b-y}z\varphi_{\sigma^2}(z)dz}  F'(y)dy\right|\label{term2_2}\\
    &+\left|\int_{a-k\sigma}^{a+k\sigma}  F'(y)\pp{\int_{a-y}^{b-y}z\varphi_{\sigma^2}(z)}dzdy\right|\label{term2_4}\\
    &+\left|\int_{b-k\sigma}^{b+k\sigma}  F'(y)\pp{\int_{a-y}^{b-y}z\varphi_{\sigma^2}(z)}dzdy\right|\label{term2_5}
\end{align}

For the term \eqref{term2_1}, we note that
\begin{equation*}
\begin{aligned}
     & \left|\int_{a+k\sigma}^{b-k\sigma}  F'(y)\pp{\int_{a-y}^{b-y}z\varphi_{\sigma^2}(z)}dzdy\right|
    \leq &\sigma^2\int_{a+k\sigma}^{b-k\sigma} F'(y)dy\leq\sigma^2.
\end{aligned}
\end{equation*}
Here we utilize that for $y\in[a+k\sigma,b-k\sigma]$, 
\begin{equation*}
\begin{aligned}
     \left| \int_{a-y}^{b-y}z\varphi_{\sigma^2}(z)dz\right| =&\left| \pp{\int^{a-y}_{-\infty}+\int_{b-y}^\infty}z\varphi_{\sigma^2}(z)dz\right| \\
     &\leq\pp{\int^{a-y}_{-\infty}+\int_{b-y}^\infty}|z|\varphi_{\sigma^2}(z)dz\\
     &\leq\pp{\int^{-k\sigma}_{-\infty}+\int_{k\sigma}^\infty}|z|\varphi_{\sigma^2}(z)dz\leq \sigma^2.
\end{aligned}
\end{equation*} 
We can bound the term \eqref{term2_2} by
\begin{equation*}
\begin{aligned}
    &\left|\pp{\int_{-\infty}^{a-k\sigma}+\int_{b+k\sigma}^\infty}\pp{\int_{a-y}^{b-y}z\varphi_{\sigma^2}(z)dz}  F'(y)dy\right|\\
    \leq&\frac{1}{2}\sigma^2 \pp{\int_{-\infty}^{a-k\sigma}+\int_{b+k\sigma}^\infty}F'(y)dy\leq \frac{1}{2}\sigma^2.
\end{aligned}
\end{equation*}
Here we utilize that for $y\leq a-k\sigma$ or $y\geq b+k\sigma$, 
\begin{equation*}
\begin{aligned}
     \left| \int_{a-y}^{b-y}z\varphi_{\sigma^2}(z)dz\right|
     \leq  \int_{a-y}^{b-y}|z|\varphi_{\sigma^2}(z)dz
     \leq \int_{k\sigma}^\infty |z|\varphi_{\sigma^2}(z)dz\leq \frac{1}{2}\sigma^2.
\end{aligned}
\end{equation*}


For the term \eqref{term2_4}, we note that
\begin{equation*}
\begin{aligned}
 &\quad\left|\int_{a-k\sigma}^{a+k\sigma}  F'(y)\pp{\int_{a-y}^{b-y}z\varphi_{\sigma^2}(z)}dzdy\right|\\
&=\left|\int_{0}^{k\sigma} F'(a+t)\pp{\int_{a-b-t}^{b-a-t}z\varphi_{\sigma^2}(z)}dzdt\right|\\
&\quad+\left|\int_{0}^{k\sigma} (F'(a-t)-F'(a+t))\pp{\int_{t}^{b-a+t}z\varphi_{\sigma^2}(z)}dzdt\right|\\
&\leq  \sigma^2\int_{0}^{k\sigma} F'(a+t)dt+\int_{0}^{k\sigma}2c_2t\sigma\sqrt{c_0}dt \\
&\leq  \sigma^2+\sqrt{c_0}c_2k^2\sigma^3.
\end{aligned}
\end{equation*}
Here we utilize that for sufficiently small $\sigma$ such that $2k\sigma\leq b-a$, 
\begin{equation*}
\begin{aligned}
     &\left|\int_{a-b-t}^{b-a-t}z\varphi_{\sigma^2}(z)\right|
    \leq& \int_{k\sigma}^\infty |z|\varphi_{\sigma^2}(z)+\int^{-k\sigma}_{-\infty} |z|\varphi_{\sigma^2}(z)\leq \sigma^2.
\end{aligned}
\end{equation*}
Similarly, we can bound the term \eqref{term2_5} by
\begin{equation*}
\begin{aligned}
 \left|\int_{b-k\sigma}^{b+k\sigma}  F'(y)\pp{\int_{a-y}^{b-y}z\varphi_{\sigma^2}(z)}dzdy\right|\leq \sigma^2+\sqrt{c_0}c_2k^2\sigma^3.
\end{aligned}
\end{equation*}
This completes the proof.

\subsection{Auxiliary results}
We introduce the following auxilary lemmas to extend the median results to quantile.
\begin{lem}\label{lem:quant_ub}
Let $\alpha \in(0,1)$. Suppose that $0<\epsilon<\min\{\alpha,1-\alpha\}/5$ and $\delta\in(0,1)$. 
For $n\geq 4\epsilon^{-2}\log(2/\delta)$ points, with probability at least $1-\delta$, $\hat \eta=X_{(\flr{\alpha n})}$ satisfies 
\begin{equation*}
    |F(\hat{\eta})-\alpha|\leq \epsilon.
\end{equation*}
\end{lem}
\begin{proof}
Assume that $\epsilon<\min\{\alpha,1-\alpha\}/5$. Consider the random variable $Z_i=1$ if $F(X_i)\leq \alpha-\epsilon$ and $0$ otherwise. Let $Z=\sum_{i=1}^n Z_i$.  
By the Chernoff bound, 
\begin{equation*}
\begin{aligned}
    \P(F(\hat \eta)\leq \alpha -\epsilon)
    \leq \P(Z\geq \alpha n)
    \leq \P(Z\geq (1+\epsilon/\alpha)\mbE[Z_1])
    \leq \exp\pp{-\frac{4\epsilon^2n}{15\alpha}}.
\end{aligned}
\end{equation*}
On the other hand, consider the random variable $Z_i'=1$ if $F(X_i)\geq \alpha+\epsilon$ and $0$ otherwise. Let $Z'=\sum_{i=1}^n Z_i'$. According to the Chernoff bound, 
\begin{equation*}
    \P(F(\hat{\eta})\geq \alpha+\epsilon)\leq \P( Z'\geq (1-\alpha)n)\leq \P( Z'\geq (1+\epsilon/(1-\alpha))\mbE[Z_i'])\leq \exp\pp{-\frac{4\epsilon^2n}{15(1-\alpha)}}.
\end{equation*}
In summary, we have
\begin{equation*}
    \P(|F(\hat \eta)- \alpha|\leq \epsilon)\leq 1-2\exp\pp{-\frac{4\epsilon^2 n}{15}}\leq 1-2\exp\pp{-\frac{\epsilon^2 n}{4}}.
\end{equation*}
Therefore, by taking $n=4\epsilon^{-2}\log(2/\delta)$, we have $\P(|F(\hat \eta)- \alpha|\leq \epsilon)\leq\delta$. This completes the proof. 
\end{proof}

\begin{lem}\label{lem:quant_ub_dist}
Assume that $F$ satisfies (B2) At $F^{-1}(\alpha)$ with $(c_1,t_1)$.  Suppose that $\epsilon\leq \min\{t_1, \min\{\alpha,1-\alpha\}/(5c_1)\}$ and $\delta\in(0,1)$. With $n\geq 4\log(2/\delta)c_1^{-2}\epsilon^{-2}$ points, we have  $\P(|X_{(\flr{\alpha n})}-F^{-1}(\alpha)|\leq \epsilon)\geq 1-\delta$.
\end{lem}
\begin{proof}
Note that $c_1\epsilon\leq c_1t_1\leq \frac{1}{5}\min\{\alpha,1-\alpha\}$. From Lemma \ref{lem:quant_ub}, with $n\geq 4\log(2/\delta)c_1^{-2}\epsilon^{-2}$, we have 
$$
\P(|F(X_{(\flr{\alpha n})})-F(F^{-1}(\alpha))|\leq c_1\epsilon)\geq 1-\delta.
$$
Let $\eta=F^{-1}(\alpha)$ and $\hat{\eta}=X_{(\flr{\alpha n})}$. If $|\hat{\eta}-\eta|>c_1t_1$, as $F$ satisfies (B2) at $\eta$ with $(c_1,c_3,t_1)$, we have
\begin{equation*}
    |F(X_{(\flr{\alpha n})})-F(F^{-1}(\alpha))|\geq \min \{|F(\eta+t_1)-F(\eta)|,|F(\eta-t_1)-F(\eta)|\}\geq c_1t_1>c_1 \epsilon,
\end{equation*}
which leads to a contradiction. If $|\hat{\eta}-\eta|\leq c_1t_1$, then, in the same manner, 
\begin{equation*}
    c_1\epsilon |F(X_{(\flr{\alpha n})})-F(F^{-1}(\alpha))|\geq c_1|\hat\eta|.
\end{equation*}
This implies that $|X_{(\flr{\alpha n})}-F^{-1}(\alpha)|\leq \epsilon$. 
\end{proof}

\section{Proofs of upper bounds for online algorithms} \label{app:ubProofs}
In this Appendix we provide the proof of \Cref{thm:gen_ada}. 
As discussed, in order to exploit the Bayesian nature of the problem, we analyze the algorithm in two parts.
First, we use the fact that our arms are drawn from a common distribution to find some $n,m$ as in \Cref{thm:gen_uni} such that the plug-in estimator $G_{n,m}$ will be an $(\epsilon/2,\delta/2)$-PAC approximation of $g(F)$.
Second, we show that our adaptive algorithm is an $(\epsilon/2,\delta/2)$-PAC approximation of $G_{n,m}$, but is able to accomplish this using significantly fewer samples.
We begin by proving the correctness of our algorithm, afterwards analyzing its sample complexity.




\subsection{Correctness}

To show correctness, we need to condition on the $n\times m$ matrix of observed samples $Y$, where $A_{i,j} = X_i + Z_{i,j}$, where $Z_{i,j}$ are i.i.d. $\CN(0,1)$. 
We couple the randomness in the analysis of the offline and online algorithms, considering our random arm pulls for both to be jointly generated, and the same matrix $Y$ fed into each algorithm. Analyzing the online algorithm, we show that it recovers the result of the offline sampling algorithm within error $\epsilon/2$ with probability at least $1-\delta/4$.

Notationally, let $\mu_1^\mathrm{uni},\dots,\mu_n^\mathrm{uni}$ be the estimates of samples of offline sampling algorithm and online sampling algorithms with given $(m,n)$. Let $N(i)$ be the number of samples for point $X_i$ from the online algorithm.

Defining $g_n$ as the $n$-sample version of the functional $g$, we proceed by showing that the output of our algorithm is close to the output of the $n,m$ offline sampling algorithm, which is close to $g(F)$.
Concretely, for our algorithm output $\hat{G}$, we have that
\begin{align*}
    \P(|g(F) - \hat{G}| \ge \epsilon)
    &\le \P(|g(F) - G_{n,m}| \ge \epsilon/2)
    +\P(|G_{n,m} - \hat{G}| \ge \epsilon/2).
\end{align*}
We see from the previous arguments regarding offline sampling that for $n,m$ as selected, we have that
\begin{equation*}
    \P(|g(F) - G_{n,m}| \ge \epsilon/2) \le \delta/2.
\end{equation*}
Now all that remains is to show that the second term is small.
We show that when $\alpha_1=\alpha_2$, the online algorithm exactly recovers the output of the offline sampling algorithm on the event that the confidence intervals hold.
When $\alpha_1 \neq \alpha_2$ (the case of the trimmed mean), we show that our estimate $\hat{G}$ is within $\epsilon/2$ of $G_{n,m}$ with probability at least $1-\delta/4$ on the event that the confidence intervals hold.

We begin by defining $\xi_1$ as the good event where our arms stay within their confidence intervals:
\begin{equation*} \label{eq:confInt}
    \xi_1 = \bigcap_{r \in \mathbb{N}, i \in [n]} \{ |\hat{\mu}_i(r) - \mu_i^\mathrm{uni}| \le \ConWid\}. 
\end{equation*}
We give a lower bound on the probability of the good event $\xi$ in the following lemma. 
\begin{lem}[Confidence intervals] \label{lem:conf}
The event $\xi_1$ defined in \eqref{eq:confInt}, where the confidence intervals of $\hat{\mu}_i^r$ about $\mu_i^\text{uni}$ hold, satisfies $\P(\xi_1) \ge 1-\delta/4$.
\end{lem}
\begin{proof}
Recall that
\begin{equation*}
    t_r \ge 8 \ConWid^{-2} \log \left(\frac{16 n \log m}{\delta}\right).
\end{equation*}

With this choice of $t_r$, we have
\begin{align*}
    \P(\xi_1^c) 
    &\le \sum_{r \in \mathbb{N},i\in[n]} \P(|\hat{\mu}_i(r) - \mu_i^\mathrm{uni}| > \ConWid)\\
    &\le n\sum_{r \in \mathbb{N}} \P(|\hat{\mu}_1(r) - \mu_1| \ge \ConWid/2) + \P(|\mu_1 - \mu_1^\mathrm{uni}| \ge \ConWid/2)\\
    &\le 4n\sum_{r \le \lceil\log(m)\rceil } \P(|\hat{\mu}_1(r) - \mu_1| \ge \ConWid/2) \\
    &\le 4n\sum_{r \le \lceil\log(m)\rceil } 2\cdot \exp\left(-2t_r(\ConWid/2)^2\right)
    \le \delta/4.
\end{align*}
This completes the proof.
\end{proof}

On this good event $\xi_1$, we show that our online algorithm exactly recovers the partitioning of arms performed by the offline sampling algorithm.
That is, for a given matrix of observations $Y$ overloading notation we can see that $g_n(Y)$, the output of the offline sampling algorithm, satisfies
\begin{equation*}
    g_n(Y) = \frac{1}{|S_n|} \sum_{i\in S_n} \mu_i^\mathrm{uni},
\end{equation*}
where $S_n$ is the set of relevant arms (i.e. those close to the boundary of $S$). 
Note that $S_n$ is a function of $g$. In the following lemma, we show that the online algorithm correctly identifies the arms in this set.
\begin{lem}\label{lem:correct_sy}
On the event $\xi_1$ we have that $\hat{S}_n$ is identical to $S_n$.
\end{lem}
The proof of this Lemma conditions on the good event where all confidence intervals hold, and shows that in this case the boundaries of $\hat{S}_n$ stay accurate throughout the course of the algorithm, and no arms are spuriously eliminated.
\begin{proof}
In this proof we focus on showing that $\hat{S}_n$ correctly partitions those elements smaller than $\mu_{(\flr{\alpha_1 n})}^\mathrm{uni}$ from those greater than this threshold.
Identical arguments hold for the analysis of $\mu_{(\flr{\alpha_2 n})}^\mathrm{uni}$, which together imply the correctness of $\hat{S}_n$.

In round $r$, we use $(r)$ to represent the corresponding values before the sampling, for example, $\hat{\mu}(r)$. Let $t_i$ be the round that the $i$-th arm is eliminated from the active set. Suppose that the algorithm ends in $T$ rounds. We denote 
\begin{equation*}
\begin{aligned}
    G(r)=&\{i:\hat{\mu}_i(r)> \hat{\mu}_{(\flr{\alpha_1n})}(r)+b_{\min\{r,t_i\}}\},\\
    L(r)=&\{i:\hat{\mu}_i(r)< \hat{\mu}_{(\flr{\alpha_1n})}(r)-b_{\min\{r,t_i\}}\},\\
    U(r)=&\{i:|\hat{\mu}_i(r)-\hat{\mu}_{(\flr{\alpha_1n})}(r)|\leq b_{\min\{r,t_i\}}\}.
\end{aligned}
\end{equation*}
From the definition of $G(r)$ and $L(r)$, it is easy to observe that $|L(T+1)| \leq \flr{\alpha_1 n}$ and $|G(T+1)| \leq n-\flr{\alpha_1n}$. We note that $|U(T+1)|  =0$ and this implies that $|G(T+1)| =n-\flr{\alpha_1n}$ and $|L(T+1)| =\flr{\alpha_1n}$. From the definition of $L(T+1)$, it consists of $\flr{\alpha_1n}$ points with minimal $\hat{\mu}_i(T+1)$, i.e.,
$$
\hat{\mu}_{(\flr{\alpha_1n})}(T+1) = \max_{i\in L(T+1)}\hat{\mu}_i(T+1).
$$
Conditioned on the good event $\xi_1$, we have
\begin{equation*}
  \hat{\mu}_{i}(T+1) -b_{t_i} \leq \mu_i^\mathrm{uni}\leq \hat{\mu}_{i}(T+1) +b_{t_i}, i\in[n].  
\end{equation*}
Note that $b_{T+1}=0$. Therefore, for arbitrary $i\in L(T+1)$ and $j\in G(T+1)$, we have
$$
\mu_i^\mathrm{uni}\leq \hat{\mu}_{i}(T+1) +b_{t_i}\leq \hat{\mu}_{(\flr{\alpha_1n})}(t) \leq \hat{\mu}_{j}(T+1)-b_{t_j}\leq \mu_j^\mathrm{uni}.
$$
Hence, the maximal element in $\{\mu^\mathrm{uni}_i\}_{i\in L(T+1)}$ is $\mu^\mathrm{uni}_{(\flr{\alpha_1n})}$, which is the $\alpha_1$-th quantile of $\{\mu^\mathrm{uni}_i\}_{i=1}^n$. 
This also implies that
\begin{equation*}
    \{i:\hat{\mu}_i(T+1)\geq \hat{\mu}_{(\flr{\alpha_1n})}(T+1)\} = \{i:\mu_i^\text{uni}\geq \mu^\mathrm{uni}_{(\flr{\alpha_1n})}\}.
\end{equation*}
On the other hand, analogously, we note that 
\begin{equation*}
   \{i:\hat{\mu}_i(T+1)\leq \hat{\mu}_{(\flr{\alpha_2n})}(T+1)\} = \{i:\mu_i^\text{uni}\leq \mu^\mathrm{uni}_{(\flr{\alpha_2n})}\}.
\end{equation*}
By combining the above two equations together, we completes the proof. 
\end{proof}

We now split our analysis into cases. 
When $\alpha_1=\alpha_2$, we see that all arms in $\hat{S}_n$ will be pulled exactly $m$ times, and so for $i \in \hat{S}_n$ we have that $\mu_i(r)=\mu_i^\mathrm{uni}$ for the final round $r$.
This implies that $\hat{G} = G_{n,m}$ for this given $Y$.

When $\alpha_1 \neq \alpha_2$, we have that some arms in the set $\hat{S}_n$ have not been pulled $m$ times; they were determined to be in $S_n$ using fewer samples, and removed from the active set as they did not require further sampling.
Thus, we will not have that $\hat{G} = G_{n,m}$.
Instead, we show that because there are so many points in $\hat{S}_n$, by sampling each of them only once and averaging the results, we obtain $\hat{G}$ which is within $\epsilon/2$ of $G_{n,m}$ with probability at least $1-\delta/4$.

\begin{lem}
On the event $\xi_1$, \Cref{alg:genalg} satisfies $\P(|G_{n,m} - \hat{G}| \ge \epsilon/2) \le \delta/4$.
\end{lem}
\begin{proof}
If $\alpha_1=\alpha_2$ then $G_{n,m} = \hat{G}$, and so the result holds trivially.

If $\alpha_1 \neq \alpha_2$, then on the good event $\xi_1$ where our confidence intervals hold, we have that our online algorithm correctly identifies $S_n(Y)$.
Then, 
\begin{align*}
\P(|G_{n,m} - \hat{G}| \ge \epsilon/2 \ |\  \xi_1)
&\le 2\P\left(\left|\frac{1}{|S_n|} \sum_{i \in S_n} \left(\tilde{\mu}_i - \mu_i^\textrm{uni}\right)\right| \ge \epsilon/2\right)\\
&= 2\P\left(\left|\mathcal{N}\left(0,\frac{1+1/m}{|S_n|}\right)\right| \ge \epsilon/2\right)\le \delta/4,
\end{align*}
as this sum is normally distributed with variance decaying with $S_n$, and so for sufficiently large $n$ we have the desired result (as when $\alpha_1\neq \alpha_2$, $|S_n| \ge \flr{n(\alpha_2-\alpha_1)}$).
\end{proof}

Thus, we see that our algorithm's output $\hat{G}$ will be close to $g(F)$ with high probability.

\subsection{Sample complexity analysis}

We now turn to bounding the sample complexity of our online algorithm. For simplicity, we overload $S_n$ in our analysis as $S_n=\text{Conv}(\{\mu_i^\mathrm{uni}:i\in S_n\})$. Useful in this analysis will be the distance from $X$ to the boundary of $S_n$ (essentially the gap of $X$), which we define as
\begin{equation*}
    \text{dist}(X,\partial S_n) = \min\left(|X - \mu_{(\flr{\alpha_1n})}^\textrm{uni}|, |X - \mu_{(\flr{\alpha_2n})}^\textrm{uni}|\right)
\end{equation*}
To this end, we provide the following Lemma:
\begin{lem}
On the good event $\xi_1$, we have that a given arm $X_i$ will be pulled $N(i)$ times where
\begin{equation*}\label{eq:complexity_raw_bound}
    N(i) \le \min\left(m, \frac{256\log \left(\frac{16n \log m}{\delta}\right)}{\left[\textnormal{dist}(X_i,\partial S_n)\right]^{2}} \right).
\end{equation*}
\end{lem}
\begin{proof}
To begin, no arm can be pulled more than $m$ times by our adaptive algorithm, due to the structure of $t_r$.
We now additionally see that by the construction of our $\ConWid$ confidence intervals, we have that 
On the good event where our confidence intervals hold, we see that an arm $X_i$ must be eliminated when $2b_r \le \text{dist}(X_i,\partial S_n)$.
Due to the iterative halving of $b_r$, this means that arm $i$ must be eliminated by round $r$ where $4b_r \ge \text{dist}(X_i,\partial S_n)$, and so $b_r^{-2} \le 16 \left[\text{dist}(X_i,\partial S_n)\right]^{-2}$.
\end{proof}

This is to say, it cannot be pulled more than $m$ times, and if it is far from the boundary of $S_n$ then it can be determined whether it is in the set or not using many fewer samples, only scaling with $\left[ \text{dist}(X,\partial S_n)\right]^{-2}$.

Thus, the total sample complexity of our online algorithm (for a given matrix of observed samples $Y$ with corresponding arm mean vector $\mu^\textrm{uni}$) is upper bounded by 
\begin{equation*} 
    \sum_{i=1}^n N(i) \le \sum_{i=1}^n \min\left(m, \frac{256\log \left(\frac{n \log m}{\delta}\right)}{\left[\textnormal{dist}(X,\partial S_n)\right]^{2}} \right).
\end{equation*}

We know by the Glivenko-Cantelli theorem that in the limit $\mu_{(\flr{\alpha_1n})}^\textrm{uni} \to F^{-1}(\alpha_1)$, but we require finite sample rates to give a useful bound.

\begin{lem} \label{lem:threshClose}
For $n=\Theta(\epsilon^{-2})$, with probability at least $1-\delta/8$ in the randomness in $Y$, we have for $i \in \{1,2\}$ that simultaneously 
\begin{equation*}
    |\mu_{(\flr{\alpha_in})}^\textrm{uni} - F^{-1}(\alpha_i)| = O\left( \sqrt{\frac{\log(1/\delta)}{m}} \right).
\end{equation*}
Denote this event as $\xi_2$.
\end{lem}

\begin{proof}
This lemma is simply a statement about the correctness of the offline sampling algorithm for estimating the $\alpha_1$/$\alpha_2$-th quantiles, which we have already proven.
\end{proof}
Note that we are not conditioning on $\xi_2$ occurring in order for our algorithm to provide the correct output; we are simply utilizing this event to bound our algorithm's sample complexity.

On the good event in \Cref{lem:threshClose} and the good event where our algorithm correctly outputs an $\epsilon$ accurate estimate and has sample complexity as in \eqref{eq:complexity_raw_bound}, we have that
\begin{align*}
    N(i) &\le \min\left(m, \frac{256 \log \left(\frac{16n \log m}{\delta}\right)}{\left[\min\left(|X_i - \mu_{(\flr{\alpha_1n})}^\textrm{uni}|, |X_i - \mu_{(\flr{\alpha_2n})}^\textrm{uni}|\right)\right]^{2}} \right)\\
    &\le 
    \begin{cases}
        m & \text{ when } \text{dist}(X,\partial S_n) \le C\sqrt{\frac{\log(1/\delta)}{m}},\\
        \frac{1024 \log \left(\frac{16n \log m}{\delta}\right)}{\left[\text{dist}(X_i,\partial S)\right]^{2}} & \text{ when } \text{dist}(X,\partial S_n) > C\sqrt{\frac{\log(1/\delta)}{m}},
    \end{cases}\\
    & \le \min\left(m, \frac{\log \left(\frac{n \log m}{\delta}\right)}{\left[\text{dist}(X_i,\partial S)\right]^{2}} \right),
\end{align*}
where $C$ is an absolute constant and $\text{dist}(X_i,\partial S)$ is defined analogously as $\text{dist}(X_i,\partial S) = \min\left(|X_i - F^{-1}(\alpha_1)|, |X_i - F^{-1}(\alpha_2)|\right)$.

We then have that our sample complexity $M$ is bounded as, conditioned on $\xi_1$ we have that
\begin{align*}
    \E[M]
    &= O\left( \E\left[\sum_{i=1}^n \min\left(m, \log \left(\frac{n \log m}{\delta}\right)\left[\text{dist}(X_i,\partial S_n)\right]^{-2} \right) \right]\right)\\
    &\le O\left( n\E\left[ \min\left(m, \log(n/\delta)\left[\text{dist}(X_i,\partial S_n)\right]^{-2} \right)  \left|\xi_2 \right. \right] + nm\P(\xi_2^c)\right)\\
    &\stepa{\le} O\left( n\log(n/\delta)\E\left[ \min\left(m, \text{dist}(X,\partial S)^{-2} \right)\right]\right).
\end{align*}
where we defined the good event $\xi_2$ as in \Cref{lem:threshClose} where our $\mu_{(\flr{\alpha_in})}^\textrm{uni}$ are within $\sqrt{\log(1/\delta)/m}$ of their distributional values, i.e. $F^{-1}(\alpha_i)$. 
We utilize the fact that $n \ge m$ to simplify the sample complexity.
(a) comes from that $\P(\xi_2^c)\leq \delta/8$ from Lemma \ref{lem:threshClose} and that for events $E$ with probability greater than $1/2$ and positive random variables $X$, we have that $\E [X|E] \le 2\E[X]$. This gives us the desired result.

\begin{thm}[Restating \Cref{thm:gen_ada}]
\Cref{alg:genalg} succeeds in estimating $g(F)$ to within accuracy $\epsilon$ with probability at least $1-\delta$, and requires at most
\begin{equation}\label{equ:up_ada_bnd}
    O\left( n\log(n/\delta)\E\left[ \min\left(m, \textnormal{dist}(X,\partial S)^{-2} \right)\right]\right)
\end{equation}
observations in expectation.
\end{thm}

\subsection{Functional-specific upper bounds} \label{app:adaCorr}
From \Cref{thm:gen_ada}, we are able to derive the upper bound sampling complexity of online algorithms in Table \ref{tab:functionals} for mean, median, maximum and trimmed mean estimation by analyzing \eqref{equ:up_ada_bnd} under the functional specific assumptions.

\subsubsection{Mean estimation}
\begin{proof}
For mean estimation, from \Cref{thm:gen_uni} we have that $n=\Theta(\epsilon^{-2})$ and $m=\Theta(1)$ is sufficient. Therefore, we have an expected sample complexity of
\begin{equation}
    \begin{aligned}
    \E[M] = &O\left( n\log(n/\delta) \E\left[ \min\left(m, \text{dist}(X,\partial S)^{-2} \right)\right]\right)\\
    =& O(n\log(n/\delta))=O(\epsilon^{-2}\log(1/\epsilon)).
\end{aligned}
\end{equation}
This completes the proof. 
\end{proof}

\subsubsection{Median estimation}
\begin{proof}
For median estimation, from \Cref{thm:gen_uni} we have that $n=\Theta(\epsilon^{-2})$ and $m=\Theta(\epsilon^{-1})$ is sufficient. We can compute that
\begin{align*}
    &\E\left[ \min\left(m, \text{dist}(X,\partial S)^{-2}\right)\right]\\
    =& O(1)\int \min\pp{\epsilon^{-1},(x-F^{-1}(0.5))^{-2}}dF(x)\\
    =& O(1) \pp{\int_{-\infty}^{F^{-1}(0.5)-\sqrt{\epsilon}}(x-F^{-1}(0.5))^{-2}dF(x)+\int_{F^{-1}(0.5)+\sqrt{\epsilon}}^{\infty}(x-F^{-1}(0.5))^{-2}dF(x)}\\
    &+ O(1)\pp{\int_{F^{-1}(0.5)-\sqrt{\epsilon}}^{F^{-1}(0.5)+\sqrt{\epsilon}}\epsilon^{-1}dx}.
\end{align*}
The first term can be bounded using integration by parts, where we note that
\begin{align*}
    &\int_{-\infty}^{F^{-1}(0.5)-\sqrt{\epsilon}}(x-F^{-1}(0.5))^{-2}dF(x)\\
    =&-\int_{-\infty}^{F^{-1}(0.5)-\sqrt{\epsilon}}F'(x)d(x-F^{-1}(0.5))^{-1}\\
    =&-\left. F'(x)(x-F^{-1}(0.5))^{-1}\right|_{-\infty}^{F^{-1}(0.5)-\sqrt{\epsilon}}+\int_{-\infty}^{F^{-1}(0.5)-\sqrt{\epsilon}}F^{(2)}(x)(x-F^{-1}(0.5))^{-1}dx\\
    \leq &\sqrt{\epsilon^{-1}}F'(F^{-1}(0.5)-\sqrt{\epsilon})+\sqrt{\epsilon^{-1}}\int_{-\infty}^{F^{-1}(0.5)-\sqrt{\epsilon}}F^{(2)}(x)\\
    \leq&2\sqrt{\epsilon^{-1}}F'(F^{-1}(0.5)-\sqrt{\epsilon})=O(\sqrt{\epsilon^{-1}}).
\end{align*}
Here we utilize that $F'(x)$ is upper bounded at $F^{-1}(0.5)-\sqrt{\epsilon}$. Similarly, we have 
$$\int_{F^{-1}(0.5)+\sqrt{\epsilon}}^{\infty}(x-F^{-1}(0.5))^{-2}dF(x)\leq O(\sqrt{\epsilon}).
$$
In summary, we have
\begin{equation*}
    \E\left[ \min\left(m, \text{dist}(X,\partial S)^{-2}\right)\right]\leq O(\sqrt{\epsilon^{-1}}),
\end{equation*}
and this implies that 
   \begin{equation*}
   \begin{aligned}
           \E[M] = &O\left( n\log(n/\delta)\E\left[ \min\left(m, \text{dist}(X,\partial S)^{-2} \right)\right]\right)
    \leq O(n\log(n/\delta)\epsilon^{-0.5})\\
    =&O(\epsilon^{-2.5} \log(1/\epsilon)).
   \end{aligned}
\end{equation*}
This completes the proof.
\end{proof}

\subsubsection{Maximum estimation}
\begin{proof}
For maximum estimation, from \Cref{thm:gen_uni} we have that $n=\Theta(\epsilon^{-\beta})$ and $m=\Theta(\epsilon^{-2})$. We note that
\begin{align*}
    &\E\left[ \min\left(m, \text{dist}(X,\partial S)^{-2}\right)\right]\\
    =& O(1)\int \min\pp{\epsilon^{-2},(x-F^{-1}(1))^{-2}}dF(x)\\
    =& O(1) \pp{\int_{-\infty}^{F^{-1}(1)-\epsilon}(x-F^{-1}(1))^{-2}dF(x)+\int_{F^{-1}(1)-\epsilon}^{\infty}\epsilon^{-2}dF(x)}.
\end{align*}
For $\beta< 2$, we can compute that
\begin{align*}
    \int_{-\infty}^{F^{-1}(1)-\epsilon}(x-F^{-1}(1))^{-2}dF(x)&=\int_{-\infty}^{F^{-1}(1)-\epsilon}F'(x)(x-F^{-1}(1))^{-2}dx\\
    &=\int_{\epsilon}^{\infty} F'(F^{-1}(1)-x)x^{-2}dx\\
    &\leq \int_{\epsilon}^{\infty} c_2\beta x^{\beta-1} x^{-2}dx=\frac{c_2\beta}{2-\beta}\epsilon^{\beta-2},
\end{align*}
and
\begin{align*}
    \int_{F^{-1}(1)-\epsilon}^{\infty}\epsilon^{-2}dF(x)=\epsilon^{-2}(1-F(F^{-1}(1)-\epsilon))\leq c_2\epsilon^{\beta-2}.
\end{align*}
This implies that 
\begin{equation*}
    \E\left[ \min\left(m, \text{dist}(X,\partial S)^{-2}\right)\right]=O(\epsilon^{\beta-2}).
\end{equation*}
As a result, we have 
$$
\E[M] = O\left( n\log(n/\delta)\E\left[ \min\left(m, \text{dist}(X,\partial S)^{-2} \right)\right]\right)=O(\epsilon^{-2}\log(1/\epsilon)).
$$
For $\beta=2$, we can compute that 
\begin{align*}
    \int_{-\infty}^{F^{-1}(1)-\epsilon}(x-F^{-1}(1))^{-2}dF(x)&=\int_{-\infty}^{F^{-1}(1)-\epsilon}F'(x)(x-F^{-1}(1))^{-2}dx\\
    &=\int_{\epsilon}^{F^{-1}(1)-F^{-1}(0)} F'(F^{-1}(1)-x)x^{-2}dx\\
    &\leq \int_{\epsilon}^{F^{-1}(1)-F^{-1}(0)} 2c_2 x x^{-2}dx=O(\log \epsilon^{-1}).
\end{align*}
Hence, we have 
$$
\E[M] = O\left( n\log(n/\delta)\E\left[ \min\left(m, \text{dist}(X,\partial S)^{-2} \right)\right]\right)
\leq O(\epsilon^{-2} \log(1/\epsilon)).
$$
For $\beta>2$, we note that
\begin{align*}
    \int_{-\infty}^{F^{-1}(1)-\epsilon}(x-F^{-1}(1))^{-2}dF(x)=&\int_{-\infty}^{F^{-1}(1)-\epsilon}F'(x)(x-F^{-1}(1))^{-2}dx\\
    &=\int_{\epsilon}^{\infty} F'(F^{-1}(1)-x)x^{-2}dx\\
    &\leq \int_{\epsilon}^{F^{-1}(1)-F^{-1}(0)} c_2\beta x^{\beta-1} x^{-2}dx=O(1).
\end{align*}
Hence, we have
$$
\E[M] = O\left( n\log(n/\delta)\E\left[ \min\left(m, \text{dist}(X,\partial S)^{-2} \right)\right]\right)
\leq O(\epsilon^{-\beta}\log(1/\epsilon)).
$$
In summary, we have 
$$
\E[M] = O\left( n\log(n/\delta)\E\left[ \min\left(m, \text{dist}(X,\partial S)^{-2} \right)\right]\right)
\leq O(\epsilon^{-\max\{\beta,2\}}\log(1/\epsilon)).
$$
This completes the proof. 
\end{proof}

\subsubsection{Trimmed mean estimation}
\begin{proof}
For trimmed mean, we note that the analysis is similar to the case of median. This gives that 
   \begin{equation*}
    O\left( n\log(n/\delta)\E\left[ \min\left(m, \text{dist}(X,\partial S)^{-2} \right)\right]\right)\leq O(n\log(n)\epsilon^{-0.5}\log(1/\epsilon))
    =O(\epsilon^{-2.5}\log^2(1/\epsilon)).
\end{equation*}
\end{proof}


\section{Proofs in Section \ref{sec:lbs}}

\subsection{Proof of Lemma \ref{lemma:Wasserstein_2}}
\begin{proof}
Denote $p_{F_1}$ and $p_{F_2}$ as the pdf of $F_1$ and $F_2$ respectively. Let $\sigma=1/\sqrt{m}$ and let $\varphi_{\sigma^2}$ be the pdf of $\mcN(0,\sigma^2)$. As $p_{\pi,F_1}=\pp{p_{F_1}*\varphi_{\sigma^2}}^{\otimes n}$ and $p_{\pi,F_2}=\pp{p_{F_2}*\varphi_{\sigma^2}}^{\otimes n}$, we have
$$
\DKL(p_{\pi,F_1}\|p_{\pi,F_2})=n\DKL(p_{F_1}*\varphi_{\sigma^2}\|p_{F_2}*\varphi_{\sigma^2}). 
$$
On the other hand, we note that $$
p_{F_1}*\varphi_{\sigma^2}(y)=\int_{-\infty}^\infty p_{F_1}(x)\varphi_{\sigma^2}(y-x)dx=\mbE_{X\sim F_1}[\varphi_{\sigma^2}(y-X)].
$$
Similarly, we have $p_{F_1}*\varphi_{\sigma^2}(y)=\mbE_{X'\sim F_2}[\varphi_{\sigma^2}(y-X')]$. 

Let $\gamma\in\Gamma$ be a coupling of $F_1$ and $F_2$. Namely, it is a joint distribution of $(X,X')$ and its marginal distribution on $X$ ($X'$) are $F_1$ ($F_2$). Then,  utilizing the convexity of KL divergence, we have
\begin{equation*}
\begin{aligned}
    &\DKL(p_{F_1}*\varphi_{\sigma^2}\|p_{F_2}*\varphi_{\sigma^2})=\DKL(\mbE_{(X,X')\sim \gamma}[\varphi_{\sigma^2}(y-X)]\|\mbE_{(X,X')\sim \gamma}[\varphi_{\sigma^2}(y-X)]) \\
    \leq &\mbE_{(X,X')\sim \gamma}\DKL(\varphi_{\sigma^2}(y-X)\| \varphi_{\sigma^2}(y-X'))
    =\mbE_{(X,X')\sim \gamma} \frac{\|X-X'\|_2^2}{2\sigma^2}.
\end{aligned}
\end{equation*}
By taking the infimum w.r.t. all possible coupling $\gamma$, we note that
\begin{equation*}
\begin{aligned}
    &\DKL(p_{\pi,F_1}\|p_{\pi,F_2})=n\DKL(F_1*\mcN(0,\sigma^2)\|F_2*\mcN(0,\sigma^2))\\
    \leq& n\inf_{\gamma\in\Gamma}\mbE_{(X,X')\sim \gamma} \frac{\|X-X'\|_2^2}{2\sigma^2}= \frac{mn}{2}\mcW(F_1,F_2)^2.
\end{aligned}
\end{equation*}
This completes the proof. 
\end{proof}

\subsection{Proof of Lemma \ref{lemma:Wasserstein_infty}}
\begin{proof}
For a given underlying distribution $F_1$ and a given algorithm $\pi$, the joint distribution of $\{(X_i,A_i,Y_i)\}_{i=1}^t$ has the following probability density function
\begin{equation*}
    p_{\pi,F_1}(\{(x_i,a_i,y_i)\}_{i=1}^t)
    =\prod_{i=1}^tp_{F_1}(x_i)p_{\pi}(a_i|(a_j,y_j)_{j=1}^{i-1})p(y_i|a_{i},\{x_j\}_{j=1}^{i}).
\end{equation*}
Thus, we can also write 
\begin{equation*}
    p_{\pi,F_1}(\{(x_i,a_i,y_i)\}_{i=1}^t) = p_{\pi}(\{(a_i,y_i)\}_{i=1}^t|\{x_i\}_{i=1}^t)\prod_{i=1}^tp_{F_1}(x_i),
\end{equation*}
where
\begin{equation*}
    p_{\pi}(\{(a_i,y_i)\}_{i=1}^t|\{x_i\}_{i=1}^t) = \prod_{i=1}^tp_\pi(a_i|(a_j,y_j)_{j=1}^{i-1})p(y_i|x_{a_i}).
\end{equation*}
Thus, the marginal distribution on $\{(a_i,y_i)\}_{i=1}^t$ follows
\begin{equation*}
\begin{aligned}
    p_{\pi,F_1}(\{(a_i,y_i)\}_{i=1}^t)     &= \int p_{\pi,F_1}(\{(x_i,a_i,y_i)\}_{i=1}^t) dz_1\dots dz_t\\
    &=\mbE_{(X_i)_{i=1}^t\sim F_1} [p_{\pi}(\{(a_i,y_i)\}_{i=1}^t|\{X_i\}_{i=1}^t)].
\end{aligned}
\end{equation*}

Let $F_2$ be a distribution different from $F_1$. We want to bound the KL divergence from $p_{\pi,F_1}(\{(a_i,y_i)\}_{i=1}^t)$ to $p_{\pi,F_2}(\{(a_i,y_i)\}_{i=1}^t)$. Let $\gamma\in \Gamma$ be a joint distribution with marginals $F_1$ and $F_2$. For simplicity, we write $\mbE_\gamma=\mbE_{(X_i,X_i')_{i=1}^t\sim \gamma}$. Utilizing the convexity of KL divergence, we note that
\begin{equation}\label{inequ:1_kl}
\begin{aligned}
    \DKL(p_{\pi,F_1} \| p_{\pi,F_2}) 
  &= \DKL(\mbE_{\gamma}[p_{\pi}(\cdot|\{X_i\}_{i=1}^t)] \| \mbE_{\gamma}[p_{\pi}(\cdot|\{X_i'\}_{i=1}^t)])\\
     &\leq\mbE_{\gamma}\DKL(p_{\pi}(\cdot|\{X_i\}_{i=1}^t) \| p_{\pi}(\cdot|\{X_i'\}_{i=1}^t)).
\end{aligned}
\end{equation}

Given the pair of underlying states $(\{X_i\}_{i=1}^t,\{X_i'\}_{i=1}^t)$, we can compute that
\begin{equation*}
    \begin{aligned}
    &\DKL(p_{\pi}(\cdot|\{X_i\}_{i=1}^t)) \| p_{\pi}(\cdot|\{X_i'\}_{i=1}^t)))\\
    =&\mbE_{\{(A_i,Y_i)\}_{i=1}^t\sim p_{\pi}(\cdot|\{X_i\}_{i=1}^t)}\bb{\sum_{i=1}^t\log\frac{p(Y_i|X_{A_i})}{p(Y_i|X'_{A_i})}}\\
    =&\mbE_{\{A_i\}_{i=1}^t\sim p_{\pi}(\cdot|\{X_i\}_{i=1}^t)}\bb{\sum_{i=1}^t\sum_{j=1}^t\mbI(A_j=i)\frac{1}{2}|X_i-X_i'|^2}\\
    =&\sum_{i=1}^t \frac{C_i(\{X_j\}_{j=1}^t)}{2}|X_i-X_i'|^2.
    \end{aligned}
\end{equation*}
Here we let $C_i(\{X_j\}_{j=1}^t)=\mbE_{\{A_i\}_{i=1}^t\sim p_{\pi}(\cdot|\{X_i\}_{i=1}^t)}\bb{\sum_{j=1}^t\mbI(A_j=i)}$. This implies that 
\begin{equation*}
\begin{aligned}
    &\mbE_{\gamma}\DKL(p_{\pi,F_1}(\cdot|\{X_j\}_{j=1}^t) \| p_{\pi,F_2}(\cdot|\{X_i'\}_{i=1}^t))    =&\mbE_{\gamma}\bb{\sum_{i=1}^t \frac{C_i(\{X_j\}_{j=1}^t)}{2}|X_i-X_i'|^2}.
\end{aligned}
\end{equation*}
We note that $\sum_{i=1}^t C_i(\{X_j\}_{j=1}^t) = t$ and 
\begin{equation*}
    \mbE_{\gamma}|X_i-X_i'|^2\leq \mathit{\rm esssup}_{(X,X')\sim \gamma}|X-X'|^2, i=1,\dots,t
\end{equation*}
This implies that 
\begin{equation*}
    \mbE_{\gamma}\DKL(p_{\pi,F_1}(\cdot|\{X_j\}_{j=1}^t) \| p_{\pi,F_2}(\cdot|\{X_i'\}_{i=1}^t)) \leq \frac{t}{2} \cdot \mathit{\rm esssup}_{(X,X')\sim \gamma}|X-X'|^2.
\end{equation*}
By taking the infinimum w.r.t. $\gamma$, we have
\begin{equation*}
\begin{aligned}
    \DKL(p_{\pi,F_1}\|p_{\pi,F_2})
        &\leq \inf_{\gamma\in \Gamma }\mbE_{\gamma}\DKL(p_{\pi,F_1}(\cdot|\{X_j\}_{j=1}^t) \| p_{\pi,F_2}(\cdot|\{X_i'\}_{i=1}^t)) \\
        &\leq t \inf_{\gamma\in \Gamma}\mathit{\rm esssup}_{(X,X')\sim \gamma}\frac{|X-X'|^2}{2}=\frac{t}{2} \mcW_\infty(F_1,F_2)^2.
\end{aligned}
\end{equation*} 
This completes the proof. 
\end{proof}

\subsection{Proof of Lemma \ref{lemma:lowerbound_maximum}}
\begin{proof}
We first give the example for the Wasserstein-2 distance. Let $G(s)$ be defined as
\begin{equation*}
    G^{-1}(s) = \left\{\begin{aligned}
    &g_1(s),& t\in[0, 2\epsilon],\\
    &s, &s\in[2\epsilon,1],\\
    \end{aligned}\right.
\end{equation*}
where $g(s)$ is a monotonic cubic interpolation satisfying that 
\begin{equation*}
    g_1(0)=\epsilon, g_1'(0)=1, g_1(2\epsilon)=2\epsilon, g_1'(2\epsilon)=1,
\end{equation*}

We note that the image of $G^{-1}$ is $[\epsilon,1]$. Therefore, the domain of $G$ is $[\epsilon,1]$. We also note that for $s\in[0,\epsilon]$, we have
\begin{equation*}
    |G(s)-s|\leq \epsilon.
\end{equation*}
Consider the following two distributions. We consider a distribution with
\begin{equation*}
    F_1(x) = 1-(1-x)^\beta
\end{equation*}
in its support $[0,1]$ and another distribution with CDF
\begin{equation*}
    F_2(x)=1-(G(1-x))^{\beta}
\end{equation*}
in its support $[0,1-\epsilon]$. We can verify that $F_1$ and $F_2$ satisfy Assumption \ref{assum:max} and $|\max(F_1)-\max(F_2)| =\epsilon$. We note that
\begin{equation*}
    F_1^{-1}(s) = 1-(1-s)^{1/\beta},    F_2^{-1}(s) = 1-G^{-1}((1-s)^{1/\beta}),
\end{equation*}
We can verify that $F_1$ and $F_2$ satisfy Assumption \ref{assum:max} and $|\max(F_1)-\max(F_2)|=\epsilon$. The Wasserstein-2 distance between $F_1$ and $F_2$ can be computed as
\begin{equation*}
\begin{aligned}
    \mcW_2(F_1,F_2)^2 &= \int_0^1 \pp{ F_1^{-1}(s) - F_2^{-1}(s) }^2 ds\\
    &= \int_0^{\epsilon^\beta}\pp{ s^{1/\beta} - G(s^{1/\beta}) }^2 ds
    \\
    &\leq  \int_0^{\epsilon^\beta}\epsilon^2 ds=\epsilon^{\beta+2}.
\end{aligned}
\end{equation*}
We then give the example for the Wasserstein-$\infty$ distance. Consider a distribution with CDF
\begin{equation*}
    F_1(x) = 1-(1-x)^\beta
\end{equation*}
in its support $[0,1]$ and another distribution with CDF
\begin{equation*}
    F_2(x)=1-(1-x+\epsilon)^{\beta}
\end{equation*}
in its support $[\epsilon,1+\epsilon]$. Let $\gamma$ be the joint distribution of $(X,X+\epsilon)$, where $X$ follows $F_1$. Then, $\gamma\in \Gamma$ is the coupling of $F_1$ and $F_2$. We can compute that
\begin{equation*}
    \mathit{\rm esssup}_{(X,X')\sim \gamma}|X-X'| =\epsilon.
\end{equation*}
This implies that $\mcW_\infty(F_1,F_2)\leq \epsilon$.

Finally, we give the example for the KL divergence. Consider two distributions with following CDFs:
\begin{equation*}
    F_{1}(x) =\begin{cases}
    \begin{aligned}
    & \frac{1-(1-x)^\beta}{1-\epsilon^\beta},&x\in[0,1-\epsilon],\\
    &0, &x\in(-\infty,0),\\
    &1, &x\in(1-\epsilon,\infty).
    \end{aligned}
    \end{cases}
\end{equation*}
\begin{equation*}
F_{2}(x) = 
\begin{cases}
    \begin{aligned}
    & 1-(1-x)^\beta,&x\in[0,1],\\
    &0, &x\in(-\infty,0),\\
    &1, &x\in(1-\epsilon,\infty).
    \end{aligned}
    \end{cases}
\end{equation*}
We can verify that $F_1$ and $F_2$ satisfy Assumption \ref{assum:max} and $|\max(F_1)-\max(F_2)|=\epsilon$. We note that
\begin{equation*}
    \sup_{x\in[0,1-\epsilon]} \frac{p_{F_1}(x)}{p_{F_2}(x)} =\frac{1}{1-\epsilon^\beta}=:\zeta.
\end{equation*}
Thus, according to the reverse Pinsker inequality, we have
\begin{equation*}
    \DKL(F_1\|F_2)\leq \frac{\log \zeta }{1-\zeta^{-1}} \DTV(F_1\|F_2).
\end{equation*}
We note that $\lim_{\epsilon\to 0} \frac{\log \zeta }{1-\zeta^{-1}}= \lim_{\epsilon\to 0}\frac{-\log(1-\epsilon^\beta)}{\epsilon^{\beta}}= 1$. For $x\in [0,1-\epsilon]$. 
$$
F_{1}(x)-F_{2}(x) =\frac{\epsilon^\beta}{1-\epsilon^\beta}F_{2}(x).
$$
Therefore, we have
\begin{equation*}
\begin{aligned}
    \DTV(F_1,F_2)= &\max_{x\in[0,1]}( F_{1}(x)-F_{2}(x))\leq \frac{\epsilon^\beta}{1-\epsilon^\beta}=O(\epsilon^{\beta}).
\end{aligned}
\end{equation*}
This implies that
\begin{equation*}
    \DKL(F_1 \| F_2)\leq O(\epsilon^{\beta}).
\end{equation*}
This completes the proof. 
\end{proof}

\section{Proof of lower bounds for median estimation} \label{app:lbProofsThresh}

We start with an auxiliary lemma to pointwisely bound the log-likelihood difference of two distributions.
\begin{lem}\label{lem:pointwise}
Consider two densities supported on $[-1,1]$ with pdf $p(x)$ and $q(x)$ such that $p(x), q(x) \ge 1/4$ and $|p(x)-q(x)|\le \varepsilon\cdot \mathbbm{1}(|x|\le \zeta)$ for all $x\in [-1,1]$, and $\int_{-\zeta}^{\zeta} x^\ell(p(x)-q(x))dx=0$ for all $\ell=0,\cdots,k$. Then for $\sigma\le 1/2$, 
\begin{align*}
    \left| \log \frac{p*\varphi_{\sigma^2}(x)}{q*\varphi_{\sigma^2}(x)} \right| \le C\varepsilon\left(\frac{\zeta}{\sigma}\right)^{k+2}, \quad \forall x\in \mathbb{R}, 
\end{align*}
where $\varphi_{\sigma^2}$ is the density function of $\mathcal{N}(0,\sigma^2)$, and $C>0$ is an absolute constant. 
\end{lem}

\begin{proof}
Write $h = p - q$, then
\begin{align*}
|h*\varphi_{\sigma^2}(x)| &= \left|\int_{-\zeta}^\zeta h(y)\varphi_{\sigma^2}(x-y)dy \right|\\
&\stepa{=} \left|\int_{-\zeta}^\zeta h(y)\varphi_{\sigma^2}(x)\sum_{\ell=0}^\infty H_\ell\left(\frac{x}{\sigma}\right) \frac{y^\ell}{\ell!\sigma^\ell}dy \right|\\
&= \left|\sum_{\ell=0}^\infty \varphi_{\sigma^2}(x)H_\ell\left(\frac{x}{\sigma}\right) \int_{-\zeta}^\zeta h(y) \frac{y^\ell}{\ell!\sigma^\ell}dy\right| \\
&\stepb{\le} 2\varepsilon\zeta\left(\frac{\zeta}{\sigma}\right)^{k+1} \cdot \sum_{\ell=k+1}^\infty \frac{\varphi_{\sigma^2}(x)}{\ell!}\left|H_\ell\left(\frac{x}{\sigma}\right)\right| \\
&\stepc{\le} 2\varepsilon\zeta\left(\frac{\zeta}{\sigma}\right)^{k+1} \cdot \sum_{\ell=k+1}^\infty \varphi_{\sigma^2}(x) \sum_{m=0}^{\lfloor \ell/2\rfloor} \frac{(|x|/\sigma)^{\ell-2m}}{(\ell-2m)!m!2^m} \\
&\le 2\varepsilon\zeta\left(\frac{\zeta}{\sigma}\right)^{k+1} \cdot \sum_{m=0}^\infty \frac{\varphi_{\sigma^2}(x)}{m!2^m}\sum_{\ell=2m}^\infty \frac{(|x|/\sigma)^{\ell-2m}}{(\ell-2m)!} \\
&\le 2\varepsilon\zeta\left(\frac{\zeta}{\sigma}\right)^{k+1} \cdot 2\varphi_{\sigma^2}(x)e^{|x|/\sigma} \\
&= 4e^{1/2}\varepsilon\zeta\left(\frac{\zeta}{\sigma}\right)^{k+1} \cdot \varphi_{\sigma^2}(|x| - \sigma), 
\end{align*}
where $H_\ell(x) = \ell!\cdot \sum_{m=0}^{\lfloor \ell/2 \rfloor} \frac{(-1)^m x^{\ell-2m}}{m!(\ell-2m)!2^m}$ is the Hermite polynomial, (c) uses its analytical form, and (a) uses its exponential generating function: 
\begin{align*}
\sum_{\ell=0}^\infty H_\ell(x)\frac{t^\ell}{\ell!} = \exp\left(xt - \frac{x^2}{2}\right). 
\end{align*}
As for the step (b), we use the assumed property of $h$ to conclude that
\begin{align*}
    \left|\int_{-\zeta}^\zeta y^{\ell}h(y)dy \right| \le 2\varepsilon\zeta^{\ell+1}\cdot \mathbbm{1}(\ell > k). 
\end{align*}

On the other hand, to lower bound the denominator $q*\phi_{\sigma^2}(x)$, we have the following observations: as $\sigma\le 1/2$, 
\begin{align*}
\int_{-1}^1 \mathbbm{1}(\varphi_{\sigma^2}(x-y) \ge \varphi_{\sigma^2}(|x| -\sigma)/e^2)dy &\ge \int_{-1}^1 \mathbbm{1}(\varphi_{\sigma^2}(x-y) \ge \varphi_{\sigma^2}(2\sigma))dy \\
&\ge \int_{-1}^1 \mathbbm{1}(0\le y\cdot \text{sign}(x)\le \sigma) dy \\
&\ge \sigma, \quad \text{if } |x|\le 2\sigma; \\
\int_{-1}^1 \mathbbm{1}(\varphi_{\sigma^2}(x-y) \ge \varphi_{\sigma^2}(|x| -\sigma))dy &\ge \int_{-1}^1 \mathbbm{1}(\sigma\le y\cdot \text{sign}(x) \le 2\sigma)dy \\
&\ge \sigma, \quad \text{if } |x| > 2\sigma. 
\end{align*}
Consequently, by Markov's inequality, 
\begin{align*}
q * \varphi_{\sigma^2}(x)
&\ge \frac{1}{4} \int_{-1}^1 \varphi_{\sigma^2}(x-y)dy \\
&\ge \frac{\varphi_{\sigma^2}(|x| - \sigma)}{4e^2}\cdot \int_{-1}^1 \mathbbm{1}(\varphi_{\sigma^2}(x-y)\ge \varphi_{\sigma^2}(|x| -\sigma)/e^2)dy \\
&\ge \frac{\varphi_{\sigma^2}(|x| - \sigma)}{4e^2}\cdot \sigma. 
\end{align*}
A combination of the above inequalities leads to
\begin{align*}
    \left|\frac{p*\varphi_{\sigma^2}(x)}{q*\varphi_{\sigma^2}(x)} - 1 \right| = \frac{|h*\varphi_{\sigma^2}(x)|}{q*\varphi_{\sigma^2}(x)} \le 16e^{5/2}\varepsilon\left(\frac{\zeta}{\sigma}\right)^{k+2},
\end{align*}
and therefore the claimed result. 
\end{proof}

Then, we introduce a lemma for constructing two distributions with matched moments.

\begin{lem}\label{lem:construct}
Let $\epsilon>0$. For any $k=1,2,\dots$, there exists a constant $b>0$ and a function $h(x)$ supported in $[-\tbd,\tbd]$ such that

    \begin{equation*}
\int_0^{\infty} h(x)dx=\epsilon, \quad \int h(x)x^i dx=0, i=0,1,\dots,2k,
\end{equation*}
For $i>k$, 
\begin{equation*}
\left|\int h(x)x^{2i-1} dx\right |\leq 2b\epsilon^{i+1/2}, \quad \int h(x)x^{2i} dx=0. 
\end{equation*}
We further have $|h(x)|\leq b\tbd$ and $h(x)$ is $b$-Lipschitz continuous. Here the constant $b$ only depends on $k$ and not on $\epsilon$.
\end{lem}
\begin{proof}
Note that we only need to prove the lemma for $\varepsilon=1$, as $h(x)=\sqrt{\varepsilon}h_0(x/\sqrt{\varepsilon})$ only properly scales the moments and preserves Lipschitzness. For $h_0$, consider the following form 
\begin{equation*}
h_0(x)=\begin{cases}
-h_1(-x), &x\in[-1,0]\\
h_1(x), &x\in[0,1]\\
0, &\text{otherwise}
\end{cases},
\end{equation*}
where $h_1(x)$ is a polynomial taking the form
\begin{equation*}
h_1(x)=\sum_{i=1}^{k+2} a_ix^i.
\end{equation*} 
Let $(a_1,\dots,a_{k+2})$ be the unique solution to the following linear system:
\begin{equation*}
\begin{aligned}
\sum_{i=1}^{k+2} a_i=0,\;\quad\sum_{i=1}^{k+2}\frac{a_i}{i+1}=1,\;\quad \sum_{i=1}^{k+2}\frac{a_i}{2j+i}=0,\; j=1,\dots,k.
\end{aligned}
\end{equation*}
Let $b:=\sum_{i=1}^{k+2} i|a_i|$. Then clearly $h_1(0)=h_1(1)=0$, $|h_1(x)|\le \sum_{i=1}^{k+2}|a_i|\le b$, and $|h_1'(x)|\le \sum_{i=1}^{k+2} i|a_i|\le b$. It remains to check the odd moments of $h_0$ (all even moments are zero by symmetry). Specifically, 
\begin{align*}
\int_0^\infty h_0(x)dx &= \int_0^\infty h_1(x)dx = \sum_{i=1}^{k+2} \frac{a_i}{i+1}=1; \\
\int_{-\infty}^\infty h_0(x)x^{2j-1}dx &= 2\int_0^\infty h_1(x)x^{2j-1}dx = 2\sum_{i=1}^{k+2} \frac{a_i}{2j+i} = 0, \quad 1\le j\le k; \\
\left|\int_{-\infty}^\infty h_0(x)x^{2j-1}dx\right| &= 2\left|\int_0^1 h_1(x)x^{2j-1}dx\right| \le 2b, \quad j>k.
\end{align*}
This completes the proof.
\end{proof}



\subsection{Proof of Lemma \ref{lemma:median_offline}}
For $\sigma\leq c\epsilon^{1/2}$, we consider two Gaussian distribution $F_1$ as the CDF of $N(0,1)$ and $F_2$ as the CDF of $N(3\epsilon,1)$. Then, $g(F_2)-g(F_1)=3\epsilon$ and $\DKL(F_1\|F_2)=O(\varepsilon^2)$. From the data-processing inequality, we have
\begin{equation*}
    \DKL(F_1*\mcN(0,\sigma^2) \| F_2*\mcN(0,\sigma^2))\leq \DKL(F_1 \| F_2)=O(\epsilon^2).
\end{equation*}

To show the lower bound $\Omega(\varepsilon^2)$, without loss of generality we assume that $F_1^{-1}(0.5)=0$ and $F_2^{-1}(0.5)\ge \varepsilon$. Proposition \ref{prop:cdf_diff} and the density lower bound in Assumption \ref{assum:med} show that
\begin{align*}
F_1*\varphi_{\sigma^2}(0) &\ge F_1(0) - \frac{c_2+1}{2}\sigma^2 = \frac{1}{2} - \frac{c_2+1}{2}\sigma^2, \\
F_2*\varphi_{\sigma^2}(0) &\le F_2(0) + \frac{c_2+1}{2}\sigma^2 \le F_2(\varepsilon) - c_1 \varepsilon + \frac{c_2+1}{2}\sigma^2 \le \frac{1}{2} - c_1 \varepsilon + \frac{c_2+1}{2}\sigma^2.
\end{align*}
Consequently, for $\sigma\le c\varepsilon^{1/2}$ with a small constant $c>0$, it holds that
\begin{align*}
F_1*\varphi_{\sigma^2}(0) - F_2*\varphi_{\sigma^2}(0) = \Omega(\varepsilon). 
\end{align*}
Therefore, Pinsker's inequality gives
\begin{align*}
     \DKL(F_1*\mcN(0,\sigma^2)\|F_2*\mcN(0,\sigma^2))&\geq \DTV(F_1*\mcN(0,\sigma^2)\|F_2*\mcN(0,\sigma^2))^2\\
     &\geq \left|F_2*\varphi_{\sigma^2}(0)- F_1*\varphi_{\sigma^2}(0)\right|^2
     =\Omega(\epsilon^2). 
 \end{align*}

For $\sigma\geq \epsilon^{1/2-\theta}$, let $F_1$ be uniform on $[-1,1]$. Then, $g(F_1)=0$. Clearly $F_1$ satisfies Assumption \ref{assum:med}. To construct $F_2$, we take the construction of $h(x)$ in Lemma \ref{lem:construct} with $k\ge 2\kappa/\theta$ and support on $[-\varepsilon_1, \varepsilon_1]$, with $\varepsilon_1 = b\varepsilon/c_2$. Here $b$ is the Lipschitz constant in Lemma \ref{lem:construct}, and $c_2$ is the smoothness constant in Assumption \ref{assum:med}. The density of $F_2$ is then taken to be
\begin{align*}
    p_{F_2}(x) = p_{F_1}(x) + \frac{c_2}{b}h(x). 
\end{align*}
As long as $\varepsilon$ is sufficiently small, we have $|p_{F_2}(x)-1/2| \le c_2/b \cdot b\sqrt{\varepsilon_1} \le 1/4$ everywhere on $x\in [-1,1]$. In other words, $p_{F_2}(x)\in [1/4, 3/4]$ on its support. Moreover, $p_{F_2}'(x)\le c_2/b\cdot b=c_2$. This shows that $F_2$ satisfies Assumption \ref{assum:med} as well. 

We first show that the median difference between $F_1$ and $F_2$ is at least $\varepsilon$. In fact, by the density upper bound $p_{F_2}(x)\le 3/4$, we have
\begin{align*}
g(F_2) \ge \frac{4}{3}\left(\frac{1}{2}-F_2(0)\right) = -\frac{4}{3}\cdot \frac{c_2}{b}\int_{-\varepsilon_1}^0 h(x)dx = \frac{4\varepsilon}{3} > \varepsilon. 
\end{align*}

Next we upper bound the KL divergence between Gaussian convolutions. By choosing $\epsilon$ sufficiently small, we have $\epsilon_1\leq \epsilon^{1-\theta}$. From Lemma \ref{lem:pointwise} and the property of $h(x)$, we immediately have
 \begin{align*}
     &\DKL(F_1*\mcN(0,\sigma^2)\|F_2*\mcN(0,\sigma^2))\\
     \leq&
     \max_{x\in\mbR}\frac{\log (p_{F_1}*\varphi_{\sigma^2}(x))}{\log(p_{F_2}*\varphi_{\sigma^2}(x))}\\
     =&O\left( \left(\frac{\sqrt{\varepsilon_1}}{\sigma}\right)^{k+2} \right) =O(\epsilon^{k\theta/2})=O(\epsilon^\kappa).
 \end{align*}


\subsection{Proof of Lemma \ref{lemma:median_online}}
We construct the same pair of distributions $(F_1, F_2)$ as in Lemma \ref{lemma:median_offline}, and it suffices to prove that when $\sigma\le c\varepsilon^{1/2}$, we have
\begin{align*}
\DKL(F_2*\mcN(0,\sigma^2) \| F_1*\mcN(0,\sigma^2)) = \Theta(\varepsilon^{1.5}). 
\end{align*}

For the upper bound, we simply use the data-processing inequality: 
\begin{align*}
\DKL(F_2*\mcN(0,\sigma^2) \| F_1*\mcN(0,\sigma^2))&\le \DKL(F_2 \| F_1) \le \chi^2(F_2 \| F_1) \\
&= \int_{-1}^1 \frac{(p_{F_1}(x) - p_{F_2}(x))^2}{p_{F_1}(x)}dx \\
&\le \left(\frac{c_2}{b}\right)^2 \int_{-1}^1 \frac{h(x)^2}{1/4}dx \\
&\le 4\left(\frac{c_2}{b}\right)^2\cdot \int_{-\sqrt{\varepsilon_1}}^{\sqrt{\varepsilon_1}} (b\sqrt{\varepsilon_1})^2 dx = O(\varepsilon^{1.5}). 
\end{align*}

For the lower bound, the same proof of Lemma \ref{lemma:median_offline} shows that $\DTV(F_1*\mcN(0,\sigma^2), F_2*\mcN(0,\sigma^2))=\Omega(\varepsilon)$. A na\"ive application of Pinsker's inequality only leads to an $\Omega(\varepsilon^2)$ lower bound on the KL divergence. A better lower bound is obtained by noticing that the signed measure $(F_1 - F_2)*\mcN(0,\sigma^2)$ is effectively supported on $[-\Theta(\sqrt{\varepsilon_1}), \Theta(\sqrt{\varepsilon_1})]$. 

To this end, recall from the proof of Lemma \ref{lemma:median_offline} that
\begin{align*}
F_1*\mcN(0,\sigma^2) (0) - F_2*\mcN(0,\sigma^2) (0) = \Omega(\varepsilon).
\end{align*}
On the other hand, Proposition \ref{prop:cdf_diff} tells that
\begin{align*}
|F_1*\mcN(0,\sigma^2) (-\sqrt{\varepsilon_1}) - F_2*\mcN(0,\sigma^2) (-\sqrt{\varepsilon_1}) | &\le |F_1 (-\sqrt{\varepsilon_1}) - F_2 (-\sqrt{\varepsilon_1}) | + (c_2+1)\sigma^2 \\
&= (c_2+1)\sigma^2.
\end{align*}
Therefore, for $\sigma\le c\varepsilon^{1/2}$ with a small enough $c>0$, we have
\begin{align*}
&\int_{-\sqrt{\varepsilon_1}}^0 |p_{F_1}*\mcN(0,\sigma^2)(x) -  p_{F_2}*\mcN(0,\sigma^2)(x)|dx \\
\ge& |F_1*\mcN(0,\sigma^2) (0) - F_1*\mcN(0,\sigma^2) (-\sqrt{\varepsilon_1}) - (F_2*\mcN(0,\sigma^2) (0)-F_2*\mcN(0,\sigma^2) (-\sqrt{\varepsilon_1}))| \\
\ge& F_1*\mcN(0,\sigma^2) (0) - F_2*\mcN(0,\sigma^2) (0) - |F_1*\mcN(0,\sigma^2) (-\sqrt{\varepsilon_1}) - F_2*\mcN(0,\sigma^2) (-\sqrt{\varepsilon_1}) | \\
= &\Omega(\varepsilon). 
\end{align*}

Let $p(x)$ and $q(x)$ be the shorthands of $p_{F_2}*\mcN(0,\sigma^2)(x)$ and $p_{F_1}*\mcN(0,\sigma^2)(x)$, respectively. The KL-divergence can be lower bounded as follows:
\begin{align*}
\DKL(F_2*\mcN(0,\sigma^2) \| F_1*\mcN(0,\sigma^2)&= \int_{-\infty}^\infty p(x)\log\frac{p(x)}{q(x)}dx \\
&= \int_{-\infty}^\infty \left(p(x)\log\frac{p(x)}{q(x)} - p(x) + q(x)\right)dx \\
&\stepa{\ge} \int_{-\sqrt{\varepsilon_1}}^0 \left(p(x)\log\frac{p(x)}{q(x)} - p(x) + q(x)\right)dx \\
&\stepb{\ge} \Omega(1)\cdot \int_{-\sqrt{\varepsilon_1}}^0 (p(x)-q(x))^2 dx \\
&\stepc{\ge} \Omega(1)\cdot \frac{1}{\sqrt{\varepsilon_1}}\cdot \left(\int_{-\sqrt{\varepsilon_1}}^0|p(x)-q(x)|dx \right)^2 \\
&= \Omega(\varepsilon^{1.5}),
\end{align*}
where
\begin{itemize}
    \item (a) is due to the non-negativity of $a\log(a/b)-a+b\ge 0$; 
    \item (b) follows from $a\log(a/b)-a+b\asymp (a-b)^2$ whenever $a,b = \Theta(1)$. The latter follows from $|p(x)-p_{F_1}(x)| =O(\sigma)=O(1)$ from Proposition \ref{prop:cdf_diff}, and similarly for $q(x)$; 
    \item (c) makes use of the Cauchy-Schwarz inequality. 
\end{itemize}
This completes the proof. 

\subsection{Proof of Theorem \ref{thm:lowerbound_median_online}}
From  Le Cam’s two-point lower bound, it is sufficient to show that the following proposition holds.
\begin{prop}\label{prop:med_ada_exm}
Suppose that $\epsilon>0$. Let $\mcF$ denote the set of distributions satisfying Assumption \ref{assum:med}. Consider an online algorithm $\pi$ with a fixed budget $t$ which outputs $\hat{G}$. Given the distribution $F\in \mcF$ of the underlying arms and the algorithm $\pi$, let $p_{\pi,F}(\{(a_i,y_i)\}_{i=1}^t)$ denote the distribution of the action-observation pairs up to the $t$-th iteration. Then, for any $\theta\in(0,1/4)$, there exist $F_1,F_2\in\mcF$ with median $g(F_1)$ and $g(F_2)$ such that $|g(F_1)-g(F_2)|\geq \epsilon$ and 
\begin{equation*}
    \DKL(p_{\pi,F_1} \| p_{\pi,F_2})\leq O(\epsilon^{2.5-\theta}t). 
\end{equation*}
\end{prop}


We start with a general log-sum inequality.
\begin{lem}\label{lem:log_sum}
Suppose that $p,q,K$ are probability density functions. Then, we have the following inequality:
\begin{equation*}
    (p*K)\log\frac{p*K}{q*K}\leq \pp{p\log\frac{p}{q}}*K.
\end{equation*}
\end{lem}

Then, we observe that the pair of distributions $(F_1, F_2)$ constructed in the proof of Lemma \ref{lemma:median_online} satisfies the following property. 
\begin{prop}\label{prop:med_ada_exm_prep}
Suppose that $\theta>0$ is a given constant. Let $\mcF$ denote the set of distributions satisfying Assumption \ref{assum:med}. Then, there exists two distribution $F_1,F_2\in\mcF$ with median $g(F_1)$ and $g(F_2)$ such that $g(F_2)=g(F_1)+\epsilon$ and they satisfy that 
$$
\DKL(F_1\|F_2)=O(\epsilon^{1.5}),
$$
Denote $\varphi_{\sigma^2}$ as the pdf of $\mcN(0,\sigma^2)$. For sufficiently small $\epsilon$ and for $\sigma$ satisfying $\sigma^2\geq \epsilon^{1-\theta}$, we further have
\begin{equation*}
    \left|\log \frac{p_{F_1}*\varphi_{\sigma^2}(x)}{p_{F_2}*\varphi_{\sigma^2}(x)}\right|\leq O(\epsilon^3).
\end{equation*}

\end{prop}
We then continue with the proof of Proposition \ref{prop:med_ada_exm}.
\begin{proof}
Consider two densities defined in Proposition \ref{prop:med_ada_exm_prep} with the parameter $\theta$. Suppose that $x_1,\dots,x_t$ are i.i.d. samples from either $F_1$ or $F_2$. Then, we note that $a_i\in\mbN$ for $i\in[t]$ and $y_i\sim \mcN(x_{a_i},1)$ for $i\in[t]$. For simplicity, we write $a^t=(a_1,\dots,a_t)$ and $y^t=(y_1,\dots,y_t)$. Hence, we can write the probability distribution of $(a^t,y^t)$ as follows
\begin{equation*}
\begin{aligned}
p_{\pi,F_1}(a^t,y^t)&=\mbE_{\{x_i\}_{i=1}^t\sim F_1}\bb{\prod_{i=1}^t\pp{p_{\pi}(a_i|a^{i-1},y^{i-1}) \frac{1}{\sqrt{2\pi}}\exp\pp{-\frac{1}{2}(y_i-x_{a_i})^2}}}\\
&=\prod_{i=1}^tp_{\pi}(a_i|a^{i-1},y^{i-1})\prod_{j\in \mathbb{N}}\mbE_{x_j\sim F_1}\bb{\prod_{i\leq t, a_i=j}\frac{1}{\sqrt{2\pi}}\exp\pp{-\frac{1}{2}(y_i-x_j)^2}}.
\end{aligned}
\end{equation*}
Let we write $n_j=\sum_{i\leq t}\mathbbm{1}(a_i=j)$ and $\bar{y}_j=\frac{1}{n_j}\sum_{i\leq t, a_i=j}y_i$. Then we can write
\begin{equation*}
\mbE_{x_j\sim F_1}\left[\prod_{i\leq t, a_i=j}\frac{1}{\sqrt{2\pi}}\exp\pp{-\frac{1}{2}(y_i-x_j)^2}\right]=\pp{p_{F_1}*K_j(\cdot, \bar{y}_j)}\cdot f_j(\{y_i\}_{a_i=j}),
\end{equation*}
Here we denote
$$
K_j(x,y)=(2\pi)^{-n_j/2}\exp(-n_j(x-y)^2/2)),
$$
and 
$$
f_j(\{y_i\}_{a_i=j})=\exp\pp{n_j\bar{y}_j^2/2-\sum_{i:a_i=j}y_j^2}.
$$
Therefore, we can write the log-likelihood ratio as
\begin{equation*}
\begin{aligned}
    \log\frac{p_{\pi,F_1}(a^t,y^t)}{p_{\pi,F_2}(a^t,y^t)}=&\sum_{j\in\mbN}\log\frac{\mbE_{x_j\sim F_1}\bb{\prod_{i\leq t, a_i=j}\frac{1}{\sqrt{2\pi}}\exp\pp{-\frac{1}{2}(y_i-x_j)^2}}}{\mbE_{x_j\sim F_2}\bb{\prod_{i\leq t, a_i=j}\frac{1}{\sqrt{2\pi}}\exp\pp{-\frac{1}{2}(y_i-x_j)^2}}}\\
    =&\sum_{j\in\mbN}\log\frac{\mbE_{x_j\sim F_1}\bb{\exp\pp{-\frac{n_j}2{(\bar{y}_j-x_j)}}}}{\mbE_{x_j\sim F_2}\bb{\exp\pp{-\frac{n_j}2{(\bar{ y}_j-x_j)}}}}.
\end{aligned}
\end{equation*}
Then, we can compute that
\begin{equation*}
\begin{aligned}
    \DKL(p_{\pi,F_1}\|p_{\pi,F_2})
    &=\int \sum_{a^T}p_{\pi,F_1}(a^T,y^T)\log\frac{p_{\pi,F_1}(a^t,y^t)}{p_{\pi,F_2}(a^t,y^t)}dy^t\\
    &=\sum_{j\in\mbN}\int \sum_{a^T} \prod_{i=1}^t p_{\pi}(a_i|a^{i-1},y^{i-1}) \prod_{k\in\mbN}\mbE_{x_k\sim F_1}\bb{\prod_{i:a_i=k}\frac{1}{\sqrt{2\pi}}\exp\pp{-\frac{1}{2}(y_i-x_k)^2}} \\
    &\qquad\times \log\frac{\mbE_{x_j\sim F_1}\bb{\exp\pp{-\frac{n_j}2{(\bar{y}_j-x_j)}}}}{\mbE_{x_j\sim F_2}\bb{\exp\pp{-\frac{n_j}2{(\bar{ y}_j-x_j)}}}}dy^t.
\end{aligned}
\end{equation*}
For $n_j\leq \epsilon^{\theta-1}$, from Proposition \ref{prop:med_ada_exm_prep}, we have
$$
\left|\log\frac{F_1*\mcN(0,1/n_i)}{F_2*\mcN(0,1/n_i)}(\bar{y}_j)\right|\leq O(\epsilon^{3}).
$$
On the other hand, for $n_j\geq \epsilon^{\theta-1}$, by utilizing Lemma \ref{lem:log_sum}, we note that
\begin{equation*}
\begin{aligned}
&\int \sum_{a^T} \prod_{i=1}^t p_{\pi}(a_i|a^{i-1},y^{i-1}) \prod_{k\in\mbN}\mbE_{x_k\sim F_1}\bb{\prod_{i:a_i=k}\frac{1}{\sqrt{2\pi}}\exp\pp{-\frac{1}{2}(y_i-x_k)^2}} \\
&\times \log\frac{\mbE_{x_j\sim F_1}\bb{\exp\pp{-\frac{n_j}2{(\bar{y}_j-x_j)}}}}{\mbE_{x_j\sim F_2}\bb{\exp\pp{-\frac{n_j}2{(\bar{ y}_j-x_j)}}}}dy^t\\
=&\iint \sum_{a^T} \prod_{i=1}^t p_{\pi}(a_i|a^{i-1},y^{i-1}) \prod_{k\neq j}\mbE_{x_k\sim F_1}\bb{\prod_{i:a_i=k}\frac{1}{\sqrt{2\pi}}\exp\pp{-\frac{1}{2}(y_i-x_k)^2}}\\
&\times f(\{y_i\}_{a_i=j}) [p_{F_1}*K_j(\cdot,\bar{y}_j)](x_j) \log\frac{[p_{F_1}*K_j(\cdot,\bar{y}_j)](x_j)}{[p_{F_2}*K_j(\cdot,\bar{y}_j)](x_j)}dx_jdy^t\\
\leq &\iint \sum_{a^T} \prod_{i=1}^t p_{\pi}(a_i|a^{i-1},y^{i-1}) \prod_{k\neq j}\mbE_{x_k\sim F_1}\bb{\prod_{i:a_i=k}\frac{1}{\sqrt{2\pi}}\exp\pp{-\frac{1}{2}(y_i-x_k)^2}}\\
&\times  f_j(\{y_i\}_{a_i=j}) p_{F_1}(x_j)\log \frac{p_{F_1}(x_j)}{p_{F_2}(x_j)}K_j(x_j,\bar{y}_j)dx_j dy^t\\
=&\iint p_{F_1}(x_j)\log \frac{p_{F_1}(x_j)}{p_{F_2}(x_j)} \sum_{a^T}\prod_{i=1}^t p_{\pi}(a_i|a^{i-1},y^{i-1})\\
&\times \int \prod_{k\neq j}\mbE_{x_k\sim F_1}\bb{\prod_{i:a_i=k}\frac{1}{\sqrt{2\pi}}\exp\pp{-\frac{1}{2}(y_i-x_k)^2}}\\
&\times \bb{\prod_{i:a_i=j}\frac{1}{\sqrt{2\pi}}\exp\pp{-\frac{1}{2}(y_i-x_j)^2}} dx_jdy^t\\
= &\int p_{F_1}(x_j)\log \frac{p_{F_1}(x_j)}{p_{F_2}(x_j)} dx_j =O(\epsilon^{1.5}).
\end{aligned}
\end{equation*}
Therefore, we have
\begin{equation*}
    \DKL(p_{\pi,F_1}\|p_{\pi,F_2})\leq \sum_{i\in\mbN}\pp{ O(\epsilon^{1.5})\mbI(n_i\geq \epsilon^{\theta-1})+O(\epsilon^{3})\mbI(n_i\leq \epsilon^{\theta-1})}.
\end{equation*}
Note that
\begin{equation*}
    t=\sum_{i\in\mbN}n_i\geq \epsilon^{\theta-1}\sum_{i\in\mbN}\mbI(n_i\geq \epsilon^{\theta-1}), 
\end{equation*}
and
\begin{equation*}
    t\geq \sum_{i\in \mbN}\mbI(n_i\leq \epsilon^{\theta-1}).
\end{equation*}
The above inequalities imply that
\begin{equation*}
    \DKL(p_{\pi,F_1}\|p_{\pi,F_2})\leq O(\epsilon^{2.5-\theta}t).
\end{equation*}
This completes the proof.
\end{proof}

\subsection{Proof of Lemma \ref{lem:log_sum}}
\begin{proof}
We note that the function $f(x)=x\log x$ is strictly convex. Suppose that $x\in\mbR$. We note that
$$
\int \frac{q(y)K(x-y)}{(K*q)(x)} dy =1.
$$
By the Jensen's inequality, we have
\begin{equation*}
    \int f\pp{\frac{p(y)}{q(y)}} \frac{q(y)K(x-y)}{(K*q)(x)} dy\geq f\pp{\int  \frac{p(y)}{q(y)}\frac{q(y)K(x-y)}{(K*q)(x)} dy}.
\end{equation*}
This implies that 
\begin{equation*}
      \frac{\pp{K*\pp{p\log\frac{p}{q}}}(x)}{(K*q)(x)}\geq \frac{(p*K)(x)}{(q*K)(x)}\log \frac{(p*K)(x)}{(q*K)(x)}.
\end{equation*}
This completes the proof.

\end{proof}

\section{Proof of Lower Bounds for trimmed mean estimation}
\subsection{Proof of Lemma \ref{lemma:lowerbound_tm_offline}}
Firstly, for $\sigma\leq C\epsilon^{1/2}$, we consider two Gaussian distribution $F_1$ as the CDF of $N(0,1)$ and $F_2$ as the CDF of $N(3\epsilon,1)$. Then, $g(F_2)-g(F_1)=3\epsilon$ and $\DKL(F_1\|F_2)=\frac{9\epsilon^2}{2}$. From the data-processing inequality, we have
\begin{equation*}
    \DKL(F_1*\mcN(0,\sigma^2)\|F_2*\mcN(0,\sigma^2))\leq \DKL(F_1\|F_2)=O(\epsilon^2).
\end{equation*}
To show the lower bound $\Omega(\varepsilon^2)$, without loss of generality we assume that $\int_{F_1^{-1}(\alpha)}^{F_1^{-1}(1-\alpha)}xdF_1(x)=0$ and $\int_{F_2^{-1}(\alpha)}^{F_2^{-1}(1-\alpha)}xdF_2(x)\le -\varepsilon$. Lemma \ref{lem:tm_noise} and the density lower bound in Assumption \ref{assum:trim} show that
\begin{align*}
&\int_{(F_1*\varphi_{\sigma^2})^{-1}(\alpha)}^{(F_1*\varphi_{\sigma^2})^{-1}(1-\alpha)}xdF_1*\varphi_{\sigma^2}(x)\geq \int_{F_1^{-1}(\alpha)}^{F_1^{-1}(1-\alpha)}xdF_1(x) - \tilde C\sigma^2 = - \tilde C\sigma^2, \\
&\int_{(F_2*\varphi_{\sigma^2})^{-1}(\alpha)}^{(F_2*\varphi_{\sigma^2})^{-1}(1-\alpha)}xdF_2*\varphi_{\sigma^2}(x)\leq \int_{F_2^{-1}(\alpha)}^{F_1^{-1}(1-\alpha)}xdF_2(x) + \tilde C\sigma^2 \leq -\varepsilon- \tilde C\sigma^2,
\end{align*}
Here $\tilde C>0$ is a constant. Consequently, for $\sigma\le c\varepsilon^{1/2}$ with a small constant $c>0$, it holds that
\begin{align*}
\int_{(F_1*\varphi_{\sigma^2})^{-1}(\alpha)}^{(F_1*\varphi_{\sigma^2})^{-1}(1-\alpha)}xdF_1*\varphi_{\sigma^2}(x) - \int_{(F_2*\varphi_{\sigma^2})^{-1}(\alpha)}^{(F_2*\varphi_{\sigma^2})^{-1}(1-\alpha)}xdF_2*\varphi_{\sigma^2}(x) = \Omega(\varepsilon). 
\end{align*}
From the algorithm for trimmed mean estimation, we can distinguish $F_1*\varphi_{\sigma^2}$ and $F_2*\varphi_{\sigma^2}$ using $O(\epsilon^{-2})$ samples. This implies that $ \DKL(F_1*\mcN(0,\sigma^2) \| F_2*\mcN(0,\sigma^2)) \geq \Omega(\epsilon^2)$.

For $\sigma\geq \epsilon^{1-\theta}$, 
let $F_1$ be uniform on $[1,2]$. Without the loss of generality, we may assume that $\sigma^2\geq 4$, $c_1\leq 0.5$, $c_3\geq 2$, $c_4\geq 2$, $c_5\leq 1$. Then, $F_1$ satisfies Assumption \ref{assum:trim}. By taking $k\geq\frac{2\kappa}{\theta}$, we can construct $h(x)$ satisfying the conditions in Lemma \ref{lem:construct} with $ \epsilon_1=4b'\epsilon$, where $b'=\max\{\frac{b}{c_2},4\}$. Consider the distribution $F_2$ with pdf $p_{F_2}(x)=p_{F_1}(x)+(b')^{-1}h(x-1-\alpha)$. By choosing $\epsilon$ sufficiently small such that $\sqrt{\epsilon_1}\leq \min\{\alpha,1-2\alpha\}$, the density function $p_{F_2}$ is supported in $[1,2]$ and $F_2^{-1}(1-\alpha)=2-\alpha$. 

As $|h(x)|\leq b\sqrt{{\epsilon_1}}$, for ${\epsilon_1}\leq \frac{1}{4c_2^2}$, we have $(b')^{-1}h(x)\leq \frac{c_2}{b}h(x)\leq c_2\tbd\leq \frac{1}{2}$. This implies that $p_{F_2}(x)\in[1/2,3/2]$ for $x\in[1,2]$. Therefore, $p_{F_2}$ is a density function.  Because $p_{F_1}'(x)=0$ for $x\in[1,2]$, $p_{F_2}(x)$ is $c_2$-Lipschitz continuous in $[1,2]$. 

Note that $F_2(1+\alpha)=F_1(1+\alpha)+(b')^{-1}\int_{-\infty}^0 h(x)dx= \alpha-4\epsilon$. Hence, it follows that
$$
4\epsilon=F_2(F_2^{-1}(\alpha))-F_2(1+\alpha)\leq 3/2 (F_2^{-1}(\alpha)-1-\alpha).
$$
This implies that  $F_2^{-1}(\alpha)\geq1+\alpha+ \frac{2}{3}\cdot 4\epsilon\geq1+\alpha+2\epsilon$. Therefore, 
\begin{equation*}
\begin{aligned}
\int_{F_2^{-1}(\alpha)}^{F_2^{-1}(1-\alpha)} xdF_2(x)&\leq \int_{1+\alpha+\epsilon}^{2-\alpha} xdF_2(x)= 2-2\alpha-2\epsilon-\int_{\epsilon}^{\sqrt{{\epsilon_1}}}(b')^{-1}h(x)dx\\
&\leq 2-2\alpha-2\epsilon+2\epsilon\sqrt{(b')^{-1}}\leq 2-2\alpha-\epsilon.
\end{aligned}
\end{equation*}
Here we utilize that $b'\geq 4$. Note that $\int_{F_1^{-1}(\alpha)}^{F^{-1}(1-\alpha)}xdF_1(x)=2-2\alpha$. This implies that
\begin{equation*}
    g(F_2)-g(F_1)\leq -\epsilon.
\end{equation*}

By choosing $\epsilon$ sufficiently small, we have $ \epsilon_1\leq \epsilon^{1-\theta/2}$. From Lemma \ref{lem:pointwise} and the property of $h(x)$, we immediately have
 \begin{align*}
     &\DKL(F_1*\mcN(0,\sigma^2)\|F_2*\mcN(0,\sigma^2))
     \leq \max_{x\in\mbR}\frac{\log (p_{F_1}*\varphi_{\sigma^2}(x))}{\log(p_{F_2}*\varphi_{\sigma^2}(x))}\\
     \leq&O(1)\cdot \frac{b\sqrt{{\epsilon_1}} \sqrt{{\epsilon_1}}^{2k+2}}{\sigma^{2k+2}}=O(\epsilon^{(2k+1)\theta/2})=O(\epsilon^\kappa).
 \end{align*}
 This completes the proof.

\subsection{Proof of Lemma \ref{lemma:tm_online}}
We construct the same pair of $F_1$ and $F_2$ as in Lemma \ref{lemma:lowerbound_tm_offline}. It is sufficient to prove that when $\sigma\le c\varepsilon^{1/2}$,
\begin{align*}
\DKL(F_2*\mcN(0,\sigma^2) \| F_1*\mcN(0,\sigma^2)) = \Theta(\varepsilon^{1.5}). 
\end{align*}
For the upper bound, we simply use the data-processing inequality: 
\begin{align*}
\DKL(F_2*\mcN(0,\sigma^2) \| F_1*\mcN(0,\sigma^2))&\le \DKL(F_2 \| F_1) \le \chi^2(F_1 \| F_2) \\
&= \int_1^2 \frac{(p_{F_1}(x) - p_{F_2}(x))^2}{p_{F_1}(x)}dx \\
&\le \left(b'\right)^2 \int_{-\sqrt{\varepsilon_1}}^{\sqrt{\varepsilon_1}} h(x)^2dx \\
&\le \left(b'\right)^2\cdot \int_{-\sqrt{\varepsilon_1}}^{\sqrt{\varepsilon_1}} (b\sqrt{\varepsilon_1})^2 dx = O(\varepsilon^{1.5}). 
\end{align*}
For the lower bound, we note that $F_2(1+\alpha)=\alpha-4\epsilon=F_1(1+\alpha)-4\epsilon$. Similar to the proof of Lemma \ref{lemma:median_online}, we have
\begin{align*}
F_1*\mcN(0,\sigma^2) (1+\alpha) - F_2*\mcN(0,\sigma^2) (1+\alpha) = \Omega(\varepsilon).
\end{align*}
Analogously, we can derive the same lower bound
$$
\DKL(F_2*\mcN(0,\sigma^2)\|F_1*\mcN(0,\sigma^2))\geq \Omega(\epsilon^{3/2}).
$$


\subsection{Proof of Theorem \ref{thm:lb_ada_tm}}
By applying the Le Cam's two point lower bound, it is sufficient to show that the following proposition holds.

\begin{prop}\label{prop:tm_ada_exm}
Suppose that $\epsilon>0$. Denote $\mcF$ as the set the set of distributions satisfying Assumption \ref{assum:trim}. Consider an online algorithm $\pi$ with a fixed budget $t$ which outputs $\hat{G}$. Given the distribution with CDF $F\in \mcF$ of the underlying arms and the algorithm $\pi$, let $p_{\pi,F}(\{(a_i,y_i)\}_{i=1}^t)$ denote the distribution of the action-observation pairs up to the $t$-th iteration. Then, for any $\theta\in(0,1/4)$, there exist $F_1,F_2\in\mcF$ with trimmed means $g(F_1)$ and $g(F_2)$ such that $|g(F_1)-g(F_2)|\geq \epsilon$ and 
\begin{equation*}
    \DKL(p_{\pi,F_1} \| p_{\pi,F_2})\leq O(\epsilon^{2.5-2\theta}t). 
\end{equation*}
\end{prop}
Similar to the proof of Proposition \ref{prop:med_ada_exm}, we start with the following proposition.
\begin{prop}\label{prop:tm_ada_exm_prep}
Suppose that $\theta>0$ is a given constant. Let $\mcF$ denote the set of distributions satisfying Assumption \ref{assum:trim}. Then, there exists two distribution $F_1,F_2\in\mcF$ with trimmed mean $g(F_1)$ and $g(F_2)$ such that $g(F_2)\leq g(F_1)-\epsilon$ and they satisfy that 
$$
\DKL(F_1\|F_2)=O(\epsilon^{1.5-\theta}),
$$
Denote $\varphi_{\sigma^2}$ as the pdf of $\mcN(0,\sigma^2)$. For sufficiently small $\epsilon$ and for $\sigma$ satisfying $\sigma^2\geq \epsilon^{1-\theta}$, we further have
\begin{equation*}
    \left|\log \frac{p_{F_1}*\varphi_{\sigma^2}(x)}{p_{F_2}*\varphi_{\sigma^2}(x)}\right|\leq O(\epsilon^3).
\end{equation*}
\end{prop}
Then, we present the proof of Proposition \ref{prop:tm_ada_exm}. 

\begin{proof}
Consider $F_1$ and $F_2$ as distributions constructed in Proposition \ref{prop:tm_ada_exm_prep}. Based on Proposition \ref{prop:tm_ada_exm_prep} and Lemma \ref{lem:log_sum}, analogously, we have
\begin{equation*}
    \DKL(p_{\pi,F_1}\|p_{\pi,F_2})\leq \sum_{i\in\mbN}\pp{ O(\epsilon^{1.5-\theta})\mbI(n_i\geq \epsilon^{\theta-1})+O(\epsilon^{3})\mbI(n_i\leq \epsilon^{\theta-1})}.
\end{equation*}
Note that
\begin{equation*}
    t=\sum_{i\in\mbN}n_i\geq \epsilon^{\theta-1}\sum_{i\in\mbN}\mbI(n_i\geq \epsilon^{\theta-1}), \quad   t\geq \sum_{i\in \mbN}\mbI(n_i\leq \epsilon^{\theta-1}).
\end{equation*}
This implies that
\begin{equation*}
    \DKL(p_{\pi,F_1}\|p_{\pi,F_2})\leq O(\epsilon^{2.5-2\theta}t).
\end{equation*}
This completes the proof.
\end{proof}


\subsection{Proof of Proposition \ref{prop:tm_ada_exm_prep}}
Consider two distributions constructed in Lemma \ref{lemma:lowerbound_tm_offline}. It is sufficient to show the bound on KL divergence and the pointwise bound. Firstly, we note that $F_1$ is uniform on $[1,2]$ and the density of $F_2$ only differs from the density of $F_1$ in $[1+\alpha-\sqrt{{\epsilon_1}},1+\alpha+\sqrt{{\epsilon_1}}]$. The difference is upper bounded by $b\sqrt{{\epsilon_1}}$. We note that for sufficiently small $\epsilon$, we have ${\epsilon}_1\leq \epsilon^{1-\theta/2}$. Therefore, we have
\begin{equation*}
\begin{aligned}
    \DKL(F_1\|F_2)
    &=-\int_{1+\alpha-\sqrt{{\epsilon_1}}}^{1+\alpha+\sqrt{\epsilon_1}} p_{F_1}\log \frac{p_{F_2}}{p_{F_1}}dx\\
    &=-\int_{1+\alpha-\sqrt{\epsilon_1}}^{1+\alpha+\sqrt{\epsilon_1}} p_{F_1}\pp{\pp{\frac{p_{F_2}}{p_{F_1}}-1}-\frac{1}{2}\pp{\frac{p_{F_2}}{p_{F_1}}-1}^2+O(\epsilon_1^{3/2})}dx\\
    &=\frac{1}{2}\int_{1+\alpha-\sqrt{\epsilon_1}}^{1+\alpha+\sqrt{\epsilon_1}}\pp{\frac{p_{F_2}}{p_{F_1}}-1}^2dx+O(\epsilon^2)\\
    &=O(\epsilon_1^{\frac{3}{2}})=O(\epsilon^{\frac{3}{2}(1-\theta/2)})=O(\epsilon^{1.5-\theta}).
\end{aligned}
\end{equation*}
We note that $F_1$ and $F_2$ have $2k$ matched moments, where $k>6/\theta$. Analogous to the result in Lemma \ref{lem:pointwise}, for all $x\in\mbR$, 
\begin{align*}
    \left| \log \frac{p_{F_1}*\varphi_{\sigma^2}(x)}{p_{F_2}*\varphi_{\sigma^2}(x)} \right|
    \le  C\frac{b\sqrt{\epsilon_1} \sqrt{\epsilon_1}^{2k+2}}{\sigma^{2k+2}} 
    \leq  C\cdot \frac{b \epsilon^{(k+1)(1-\theta/2)}}{\epsilon^{(k+1)(1-\theta)}}
    = O(\epsilon^{(k+1)\theta/2})
    \leq O(\epsilon^3).
\end{align*}
Here $C>0$ is an absolute constant. This completes the proof. 

\end{document}